%% file: JSTSP2016_SingleCol_Vakili_Zhao_Arxiv_v8142017.tex
\documentclass[12pt,onecolumn,journal,draftclsnofoot,singlespace]{IEEEtran}
\makeatletter
\def\ps@headings{%
\def\@oddhead{\mbox{}\scriptsize\rightmark \hfil \thepage}%
\def\@evenhead{\scriptsize\thepage \hfil \leftmark\mbox{}}%
\def\@oddfoot{}%
\def\@evenfoot{}}
\makeatother 
\usepackage{float}
\usepackage{times,amsmath,amssymb,graphicx,epsf,psfrag,epsfig,cite,color}
\usepackage{graphicx}
\usepackage{caption,subcaption}
\usepackage{multirow}
\usepackage{bbm,amsthm}

\input macros

\def\scalefig#1{\epsfxsize #1\textwidth}
\def\nn{{\nonumber}}

\newtheorem{lemma}{Lemma}
\newtheorem{theorem}{Theorem}

\begin{document}

\title{\huge Risk-Averse Multi-Armed Bandit Problems under\\ Mean-Variance Measure}

\author{\small Sattar Vakili, Qing Zhao\\
\small School of Electrical and Computer Engineering,\\ Cornell University, Ithaca, NY 14850\\
\{sv388, qz16\}@cornell.edu}
\maketitle

\begin{abstract}The multi-armed bandit problems have been studied mainly under the measure of \emph{expected} total reward accrued over a horizon of length $T$. In this paper, we address the issue of \emph{risk} in multi-armed bandit problems and develop parallel results under the measure of mean-variance, a commonly adopted risk measure in economics and mathematical finance. We show that the model-specific regret and the model-independent regret in terms of the mean-variance of the reward process are lower bounded by $\Omega(\log T)$ and $\Omega(T^{2/3})$, respectively. We then show that variations of the UCB policy and the DSEE policy developed for the classic risk-neutral MAB achieve these lower bounds.  \end{abstract}

\footnotetext{Copyright (c) 2016 IEEE. Personal use of this material is permitted. However, permission to use this material for any other purposes must be obtained from the IEEE by sending a request to pubs-permissions@ieee.org}

\footnotetext{This work was supported in part by the Army Research Office under Grant
W911NF-12-1-0271 and by the National Science Foundation under Grant 1549989. Part of this work was presented at IEEE International Conference on Acoustics, Speech and Signal Processing (ICASSP), April, 2015 and Allerton Conference on Communication, Control, and Computing, September, 2015. }

\section{Introduction}
\subsection{Risk-Neutral MAB}
Multi-armed bandit (MAB) is a class of online learning and sequential
decision-making problems under unknown models. An abstraction of this class of problems involves a slot machine with
$K$ independent arms and a single player. At each time, the player
chooses one arm to play and obtains a random reward drawn i.i.d.
over time from an unknown distribution specific to the chosen arm. The design objective is a
sequential arm selection policy that maximizes the total expected
reward over a horizon of length $T$ by striking a balance between earning immediate reward (exploitation) and learning the unknown reward models of all arms (exploration). The performance of an arm selection policy is measured by \emph{regret} defined as the expected cumulative reward loss over the entire time horizon against an omniscient player who knows the reward models and always plays the best arm~\cite{Lai&Robbins85AAM}. In their seminal work~\cite{Lai&Robbins85AAM}, Lai and Robbins showed
that the minimum regret achievable by any consistent policy is $\Omega(\log T)$. Several online learning policies exist in the literature that achieve the optimal regret order under various assumptions on the reward models (see \cite{Lai&Robbins85AAM,AgrawalEtal95AAP,Auer&etal02ML,DSEE,UCBRev}).

The above results were obtained under the so-called model-specific setting which focuses on the class of consistent (i.e., uniformly good) policies and characterizes their regret performance specific to the given set of reward distributions. The model-specific regret thus typically depend on certain statistics of the model such as the Kullback-Leiber (KL) divergence between the reward distributions and the gap in mean value from a suboptimal arm to the optimal arm.
Subsequent studies also considered the model-independent setting in which the performance of a learning policy is measured against the worst-case assignment of the reward distributions.
An $\Omega(\sqrt T)$ lower bound on the model-independent regret can be concluded from the lower bound results in~\cite{BR} and also from the lower bound results on the non-stochastic MAB problem studied in~\cite{NonSto}. A modification of the UCB policy was shown to achieve the optimal model-independent regret order~\cite{UCBRev}. A summary of the main results on risk-neutral MAB problems is given in the first column of Table~\ref{tabl1}. Readers are also referred to the comprehensive survey in~\cite{Survey}.

\subsection{Risk-Averse MAB}
\par The classic MAB formulation targets at maximizing the \emph{expected} return of an online learning policy. In many applications, especially in economics and finance, a player may be more interested in reducing the uncertainty (i.e., risk) in the outcome, rather than achieving the highest ensemble average. The focus of this paper is to develop results on risk-averse MAB, parallel to those on the classic risk-neutral MAB problems as summarized in the first column of Table~\ref{tabl1}.

The notions of risk and uncertainty have been widely studied, especially in economics and mathematical finance. A commonly adopted risk measure is \emph{mean-variance}~\cite{Marko}. Introduced by Markowits in 1952, mean-variance is particularly favored for portfolio selection in finance~\cite{Finance}. Specifically, the mean-variance $\xi(X)$ of a random variable $X$ is given by
\begin{eqnarray}\label{def1}
\xi(X)=\sigma^2(X)-\rho\ \mu(X),
\end{eqnarray}
where $\sigma^2(X)$ and $\mu(X)$ are, respectively, the variance and the mean of $X$, the coefficient $\rho>0$ is the risk tolerance factor that balances the two objectives of high return and low risk. The definition of mean-variance can be interpreted as the Lagrangian relaxation of the constrained optimization problem of minimizing the risk (measured by the variance) for a given expected return or maximizing the expected return for a given level of risk.

In~\cite{Sani}, a risk-averse MAB formulation based on the metric of mean-variance of observations was studied. Specifically, let $\pi(t)$ ($t=1,2,\ldots, T$), denote the arm played by a policy $\pi$ and $X_{\pi(t)}(t)$ the observed reward at time $t$. The cumulative mean-variance of the observed reward process is given by\footnote{Notice that the cumulative mean-variance of the observed reward process is considered in contrast to a normalized version divided by~$T$ as considered in~\cite{Sani}. This definition facilitates the comparison with the risk-neutral MAB results given in Table~\ref{tabl1}.}
\begin{eqnarray}\label{MVPi}
{\xi}_\pi(T)= \mathbb{E}\bigg[ \sum_{t=1}^{T} [ (X_{\pi(t)}(t) - \frac{1}{T}\sum_{t=1}^{T} X_{\pi(t)}(t))^2-\rho X_{\pi(t)}(t)]\bigg],
\end{eqnarray}
where the first term inside the expectation corresponds to the cumulative empirical variance and the second term the cumulative empirical mean. The objective is a learning policy that minimizes ${\xi}_\pi(T)$.
In this risk-averse model, the time variations in the observed reward process are considered as risk (see Sec.~\ref{motif} for motivating applications for this metric). Similar to risk-neutral MAB, regret is defined as the performance loss with respect to the optimal policy under a known model.

While conceptually similar, regret in terms of mean-variance of observations differs from that in total expected reward in several major aspects that complicate the analysis of the lower bounds and algorithm performance. First, under the measure of expected reward, the optimal policy under a known model is to play the arm with the highest mean value over the entire horizon. Under the measure of mean-variance, however, the optimal policy in the known model case is not necessarily a single-arm policy (as shown in Sec.~\ref{SecReg}) and is in general intractable. Second, under the measure of mean-variance, regret can no longer be written as the sum of certain properly defined immediate performance loss at each time instant.
More specifically, under the measure of mean-variance of observations, the contribution from playing a suboptimal arm at a given time $t$ to the overall
regret cannot be determined without knowing the entire sequence of decisions and observations. Third, regret in mean-variance involves higher order statistics of the random time spent on each arm. These fundamental differences in the behavior of regret are what render the
problem difficult and call for different techniques from that used in risk-neutral MAB problems.

The focus of \cite{Sani} was on developing learning policies. Specifically, two learning policies were developed and analyzed. The first one is a variation of the UCB policy (referred to as MV-UCB), and the second a variation of the DSEE policy (referred to as MV-DSEE) originally developed in~\cite{Auer&etal02ML} and~\cite{DSEE} for risk-neutral MAB.  It was shown\footnote{
In~\cite{Sani}, regret was defined comparing to the optimal single-arm policy that as we show in this paper is not necessarily the optimal policy under a known model. However, we show that the difference between regret with regard to the optimal single-arm policy and the one with regard to the optimal policy is sufficiently small that preserves the order of the results (See Sec.~\ref{SecReg}).
Also, in~\cite{Sani}, a weaker regret definition, referred to as the pseudo regret, was considered. It was shown that the pseudo regret of MV-UCB was $O(\log^2(T))$. However, since the gap between pseudo regret and the strict regret is in the order of $O(\sqrt T)$ (see Lemma~1 in~\cite{Sani}), the analysis in~\cite{Sani} only showed an $O(\sqrt T)$ regret order of MV-UCB.  We also point out that the two types of regret (model-specific vs. model-independent) were not distinguished in~\cite{Sani}. From their analysis, however, it is clear that the result on MV-UCB was in terms of model-specific regret while the result on MV-DSEE was in terms of model-independent regret.} that the model-specific regret growth rate of MV-UCB was $O(\sqrt{T})$ and the model-independent regret growth rate of MV-DSEE was $O(T^{2/3})$.

Major questions that remain open are whether the $\sqrt{T}$ model-specific regret order and the $T^{2/3}$ model-independent regret order are the best one can hope for and whether the significant gaps in regret growth rate between risk-neutral MAB and risk-averse MAB (from $\log T$ to $\sqrt{T}$ in model-specific regret and from $\sqrt{T}$ to $T^{2/3}$ in model-independent regret) are inherent to the risk measure of mean-variance of observations. As shown in this paper, the answer to these questions is negative in terms of model-specific regret and positive in terms of model-independent regret. Specifically, for model-specific regret, we establish an $\Omega(\log T)$ lower bound on the regret growth rate and provide a finer analysis of MV-UCB showing its $O(\log T)$ regret performance (in contrast to the $O(\sqrt{T})$ result given in~\cite{Sani}). In other words, the best achievable model-specific regret order remains to be logarithmic as in the risk-neutral MAB. In terms of model-independent regret, we show that the minimum regret growth rate is $\Omega(T^{2/3})$. Thus, the analysis of MV-DSEE given in~\cite{Sani} is tight.  We thus complete in this paper parallel results on risk-averse MAB under the measure of mean-variance of observations as summarized in the second column of Table~\ref{tabl1}.

\begin{table}[]
\begin{center}
\begin{tabular}{ ||c|c|c|c|| }
\hline\hline
& & &  \\
 &  & Risk-neutral MAB &  Risk-averse MAB \\
 & & & \\
\hline
\multirow{6}{5.7em}{Model-Specific~~}  & & &  \\
 & ~~Lower Bounds~~ & $\Omega(\log T)$ & {$\Omega(\log T)$}\\
 & & \cite{Lai&Robbins85AAM} &  \\
\cline{2-4}

 &  & & \\
 &Order-Optimal & \cite{Lai&Robbins85AAM,AgrawalEtal95AAP,Auer&etal02ML,DSEE} & {MV-UCB} \\
 & ~~~Policies~~~ & & \\
 & & &  \\

\hline

\multirow{6}{5.7em}{Model-Independent}  & & &  \\
 & Lower Bounds & $\Omega(\sqrt T)$ & {$\Omega( T^{2/3})$} \\
 & &\cite{NonSto,BR} &  \\

\cline{2-4}

 && & \\
&  Order-Optimal& \cite{UCBRev}& {~~~MV-DSEE~~~} \\
 &  Policies& & { \cite{Sani}}  \\
 & & &  \\
\hline\hline

\end{tabular}
\end{center}
\caption{Summary of results on risk-neutral and risk-averse MAB.\label{tabl1}}
\end{table}

\subsection{Motivating Applications}\label{motif}
Mean-variance is a well accepted risk measure whose quadratic scaling captures the natural inclination toward less risky options when the stakes are high. Studies have confirmed such risk-averse behaviors in investors (e.g. see \cite{Ku10}).

In the classic application of mean-variance to portfolio selection, the objective is a joint optimization of risk and return for a portfolio over a particular period of time. This guarantees a high expected return and a low variation in the outcome. A similar approach is also taken for intertemporal returns of assets. Specifically, the objective is to guarantee high average return and low variations over time~\cite{Volatility}. Such intertemporal variations are commonly referred to as \emph{volatility} in finance literature and measured by the sample variance of the return process.
The metric of mean-variance of the reward process studied in this paper as well as in~\cite{Sani} and~\cite{Even} precisely captures the objective of low volatility and high expected return. Another motivating application is clinical trial, where, besides obtaining high average return, it is desirable to avoid high variations in the treatment outcomes for different patients~\cite{Sani}.

Another formulation of risk-averse MAB is to consider the mean-variance of the total return at the end of the time horizon where the objective is to minimize the ensemble variations of the total return.
These two measures of mean-variance of the reward process and mean-variance of the total reward are suitable for different applications. For example, in the return of a financial security, the fluctuations over time are to be avoided as ``risk for financial security''~\cite{Intro}, while in a retirement investment, one might be more interested in the variation of the final return and less sensitive to the fluctuations in the intermediate returns. Some initial results on MAB under mean-variance of the total reward can be found in our preliminary study reported in~\cite{VZAllerton15}.

\subsection{Related Work}
\par There is a large body of work on risk-neutral MAB problems under different variations and for various applications, including clinical trials, internet advertising, web
search, and communication networks
(see~\cite{{KQ1},{MulPla},{Combinatorial},{LearningChanging},{HT1},{Monta},{Mingi},{Sali}}
and references therein).
MAB has also been applied to a variety of scenarios in finance and economics (see, for example, a comprehensive  survey in~\cite{econSurvey}).

There are relatively few studies on risk-averse MAB.
In an initial work on this topic, a sequential risk-averse problem using the measure of mean-variance of observations was formulated in~\cite{Even}. Different from this paper and \cite{Sani} that consider a stochastic formulation, \cite{Even} adopted the so-called non-stochastic full-information framework and established a negative result showing the infeasibility of sublinear regret.

There are a couple of results on risk-averse MAB under different risk measures. In \cite{GRAMAB},  the quality of an arm was measured by a general function of the mean and the variance of the random variable. This study, however, is closer to the risk-neutral MAB problems than to the problem studied in this paper. The reason is that under the model of \cite{GRAMAB}, regret remains to be the sum of the immediate performance loss at each time instant. As discussed earlier, regret in mean-variance of observations is no longer summable over time.

In~\cite{Safety,VZAllerton15}, MAB under the measure of \emph{value at risk}, which defines the minimum value of a random variable at a given confidence level, was studied. In~\cite{Safety},
learning policies using the measure of
conditional value at risk were developed. However, the performance guarantees were still within the risk-neutral MAB framework (in terms of the loss in the expected total reward) under the assumption that the best arm in terms of the mean value is also the best arm in terms of the conditional value at risk. In our recent work~\cite{VZAllerton15}, we considered risk-averse MAB under the measure of value at risk of the total reward and developed learning policies that offer poly-log regret performance. Another risk measure for MAB problems was considered in~\cite{Mill} in which the logarithm of moment generating function was used as a risk measure and high probability bounds on regret were obtained.

There are also a couple of studies, while not directly addressing risk-averse MAB, offering relevant results from different perspectives.
In \cite{VAR}, the sample complexity of both mean-variance and value at risk for single-period and multi-period decision making was studied. In~\cite{FuBa}, the problem of identifying the best arm in terms of different risk measures assuming
the existence of an efficient risk estimator was considered. Identifying the best arm is, however, different from an MAB formulation due to the absence of the tradeoff between exploration and exploitation which is at the heart of online learning problems.
Readers are also encouraged to read the work by Audibert~\etal~\cite{Audi} on the deviation of regret from its expected value.

\section{Problem Formulation and Preliminaries}\label{SecTwo}

Consider a $K$-armed bandit and a single player. At each time~$t$,
the player chooses one arm to play. Playing arm $i$ yields a
random reward $X_i(t)$ drawn i.i.d. from an unknown distribution $f_i$.
Let $\Fc=(f_1,\cdots,f_K)$ denote the set of the unknown
distributions. An arm selection policy $\pi$ specifies a function at each time $t$ that maps from the
player's observation and decision history to the arm to play at time $t$. Let $\{X_{\pi(t)}(t)\}_{t=1}^T$ denote the random reward sequence under policy $\pi$. The cumulative mean-variance $\xi_\pi(T)$ of the reward sequence is given in~\eqref{MVPi}.

The performance of policy $\pi$ is measured by regret $R_\pi(T)$ defined as the increase in cumulative mean-variance over a given horizon of length $T$ as compared to the optimal policy $\pi^*$ under a known model. (See Sec.~\ref{SecReg} for a detailed discussion on $\pi^*$)
\begin{eqnarray}\label{RPI1}
R_\pi(T) =  {\xi}_\pi(T)-  {\xi}_{\pi^*}(T).
\end{eqnarray}

\subsection{Notations}
Throughout the paper,~$*$ is used to indicate the arm that has the smallest mean-variance. If there are more than one arm with the smallest mean-variance value, one of them is chosen as~$*$.
Let $\Gamma_{i,j}=\mu_i-\mu_j$ and $\Delta_i=\xi_i-\xi_*$ denote, respectively, the difference between the mean values of arm $i$ and $j$, and the difference between the mean-variance of arm $i$ and the arm with the smallest mean-variance. Let $\Delta=\min_{i\neq*}\Delta_i$, $\Gamma=\max_{i}|\Gamma_{i,*}|$, $\sigma_{\max}=\max_{i}\sigma_i$ and $\mu_{\max}=\max_{i}\mu_i$.

The following notations are used for the sample mean, the sample variance, and the sample mean-variance of the random reward sequence from arm $i$ under a given policy $\pi$:
\begin{eqnarray}\nn
\overline{\mu}_i(t)&=&\frac{1}{\tau_i(t)}\sum_{s=1}^{\tau_i(t)}X_i(t_i(s)),\\\nn
\overline{\sigma^2}_i(t)&=&\frac{1}{\tau_i(t)}\sum_{s=1}^{\tau_i(t)}(X_i(t_i(s))-\overline{\mu}_i(t))^2,\\\nn
\overline{\xi}_i(t)&=&\overline{\sigma^2}_i(t)-\rho\overline{\mu}_i(t),
%\\\nn
%\overline{\Gamma}_{i,j}(t)&=&\overline{\mu}_i(t)-\overline{\mu}_j(t),
\end{eqnarray}
where $t_i(s)$ denotes the time instant corresponding to the $s$'th observation from arm $i$ and $\tau_i(t)$ denotes the number of times arm $i$ has been played up to time $t$. Note that these quantities depend on the policy $\pi$, which is omitted for simplicity. The time argument may also be omitted when it is clear from the context. The use of the biased estimator for the variance is for the simplicity of the expression. The results presented in this work remain the same with the
use of the unbiased estimator with $\tau_i(t)$ replaced by $\tau_i(t)-1$ in the expression of $\overline{\sigma^2}_i(t)$.

The KL-divergence between two distributions $f$ and $g$ is given by
\begin{eqnarray}
I(f,g)=\mathbb{E}_f[\log\frac{f(X)}{g(X)}],
\end{eqnarray}
where $\mathbb{E}_f$ denotes the expectation operator with respect to $f$.

In the proofs, the notation $\mathbb{E}[X, \mathcal{E}]$ for a random variable $X$ and an event $\mathcal{E}$ is equivalent to $\mathbb{E}[X\mathbb{I}_\mathcal{E}]$, where $\mathbb{I}$ is the indicator function.

\subsection{Concentration of the Sample Mean-Variance}
\par We assume that $(X_i-\mu_i)^2-\sigma_i^2$ ($i=1,\ldots,K$) for all arms have a sub-Gaussian distribution. Recall that a real-valued random variable $X$ is called sub-Gaussian if it
satisfies the following~\cite{SubG},
\begin{eqnarray}\label{zeta}
\mathbb{E}[e^{uX}]\le e^{\zeta_0 u^2/2}
\end{eqnarray}
for some constant $\zeta_0>0$.

We establish in Lemma 1 a concentration result on the sample mean-variance, which plays an important role in regret analysis.  This result is similar to the Chernoff-Hoeffding bound on the concentration of the sample mean for sub-Gaussian random variables~\cite{Chernoff} and locally sub-Gaussian random variables (i.e., light-tailed random variables with zero mean) that satisfy \eqref{zeta} only locally for $|u|\le u_0$ for some positive $u_0$~\cite{UCBEx}. The Chernoff-Hoeffding bound provides an upper bound on the probability of a given deviation of the sample mean from the true mean as follows. Let $\overline{\mu}_s$ be the sample mean of a random variable $X$ obtained from $s$ i.i.d. observations. Let $\mu=\mathbb{E}[X]$ and assume that $(X-\mu)$ has a sub-Gaussian distribution. We have, for all $a\in(0,\frac{1}{2\zeta_0}]$,
\begin{eqnarray}\label{CHBnd}
\left\{
\begin{array}{ll}
\mathbb{P}[\overline{\mu}_s-\mu<-\delta] &\le \exp(-{as\delta^2}),\\
\mathbb{P}[\overline{\mu}_s-\mu>\delta] &\le  \exp(-{as\delta^2}),\label{nine0}
\end{array}\right.
\end{eqnarray}
where $\zeta_0$ is given in~\eqref{zeta}. For locally sub-Gaussian random variables,~\eqref{CHBnd} holds locally for $\delta<\xi_0 u_0$~\cite{UCBEx}.
For $\delta\ge\xi_0 u_0$, we have the following concentration inequalities for locally sub-Gaussian random variables~\cite{UCBEx}:
\begin{eqnarray}\label{CHBnd2}
\left\{
\begin{array}{ll}
\mathbb{P}[\overline{\mu}_s-\mu<-\delta] &\le \exp(-{\frac{u_0s\delta}{2}}),\\
\mathbb{P}[\overline{\mu}_s-\mu>\delta] &\le  \exp(-{\frac{u_0s\delta}{2}}).\label{nine01}
\end{array}\right.
\end{eqnarray}
In the following lemma, we extend the Chernoff-Hoeffding bound to the sample mean-variance. Similar concentration inequalities for mean-variance were given in~\cite{Sani} and~\cite{VAR} for random variables with bounded support.
\begin{lemma}\label{LemmaChernoff}
Let $\overline{\xi}_s$ be the sample mean-variance of a random variable $X$ obtained from $s$ i.i.d. observations. Let $\mu=\mathbb{E}[X]$, $\sigma^2=\mathbb{E}[(X-\mu)^2]$, and assume that $(X-\mu)^2-
\sigma^2$ has a sub-Gaussian distribution, i.e.,
\begin{eqnarray}\nn
\mathbb{E}[e^{u((X-\mu)^2-\sigma^2)}]\le e^{\zeta_1 u^2/2}
\end{eqnarray}
for some constant $\zeta_1>0$.
As a result $X-\mu$ has a sub-Gaussian distribution, i.e.,
\begin{eqnarray}\nn
\mathbb{E}[e^{u(X-\mu)}]\le e^{\zeta_0 u^2/2}.
\end{eqnarray}
Let $\zeta=\max\{\zeta_0,\zeta_1\}$. We have, for all constants $a\in(0,\frac{1}{2\zeta}]$ and $\delta>0$,
\begin{eqnarray}\label{narnia1}
\mathbb{P}[\overline{\xi}_s-\xi(X)>\delta] &\le 2 \exp(-\frac{as\delta^2}{(1+\rho)^2}),
\end{eqnarray}
and for all $a\in(0,\frac{1}{2\zeta}]$ and $\delta\in(0,2+\rho]$,
\begin{eqnarray}\label{narnia2}
\mathbb{P}[\overline{\xi}_s-\xi(X)<-\delta] &\le 2 \exp(-\frac{as\delta^2}{(2+\rho)^2}).
\end{eqnarray}

\end{lemma}

\hspace{1em}

\noindent{\it Proof:} See Appendix~A.

In this paper, we focus on sub-Gaussian reward distributions. The main results hold for locally sub-Gaussian distributions with minor modifications similar to~\cite{UCBEx} and as commented in the paper.

\section{The Known Model Case}\label{SecReg}

In this section, we study the case where all arm distributions are known. This defines the benchmark performance in the regret definition given in~\eqref{RPI1}. We first show through a counter example that playing the arm $*$ that has the smallest mean-variance may not be optimal. This presents a major difficulty in regret analysis given that explicit characterizations of the optimal policy $\pi^*$ for the known model case are in general intractable. Our approach is to bound the performance gap between $\pi^*$ and the optimal single-arm policy $\widehat{\pi}^*$ (i.e., playing arm $*$ all through), which allows us to analyze the order of the regret defined with respect to $\pi^*$ by analyzing $\widehat{\pi}^*$.

To see that $\widehat{\pi}^*$ may not be optimal, the key is to notice that the variance term (i.e., the first term on the right-hand side of~\eqref{MVPi}) in the cumulative mean-variance is with respect to the sample mean calculated from rewards obtained from all arms. When the remaining time horizon is short and the current sample mean is sufficiently close to the mean value of a suboptimal arm $j\neq *$, it may be more rewarding (in terms of minimizing the mean-variance) to play arm $j$ rather than arm $*$. Consider a concrete example with two Gaussian-distributed arms with parameters
$\mu_1=0$, $\mu_2=1$, $\sigma_1^2=1$, $\sigma_2^2=2.1$. Let $\rho=1$ and $T=2$. It is easy to see that $\xi_1=1$ and $\xi_2=1.1$, and the optimal single-arm policy $\widehat{\pi}^*$ is to always play arm 1, yielding a cumulative mean-variance of $\xi_{\widehat{\pi}^*}(t)=1$. Consider a policy $\pi$ with $\pi(1)=1$ and $\pi(2)=\mathbb{I}_{X_1(1)<0.5}+2\mathbb{I}_{X_1(1)\ge0.5}$. It can be shown that $\xi_\pi(T)<0.7$, demonstrating the sub-optimality of $\widehat{\pi}^*$.

The above example also gives a glimpse of the complexity in finding $\pi^*$ for a general problem. To circumvent this difficulty, our approach is to show that $\widehat{\pi}^*$ is a good proxy of $\pi^*$ with a performance loss upper bounded by a constant for large $T$. We can then obtain regret bounds through $\widehat{\pi}^*$.

Recall that regret $R_\pi(T)$ in~\eqref{RPI1} is defined with respect to $\pi^*$. Using $\widehat{\pi}^*$ as the benchmark, we define a proxy regret $\widehat{R}_{\pi}(T)$ as
\begin{eqnarray}\label{RPI2}
\widehat{R}_\pi(T) =  {\xi}_\pi(T)-  {\xi}_{\widehat{\pi}^*}(T).
\end{eqnarray}
Our objective is to bound the difference between $R_\pi(T)$ and $\widehat{R}_{\pi}(T)$. To do this, we first
derive in Lemma~\ref{RegEx} a closed-form expression of $\widehat{R}_\pi(T)$ as a function of the number of times $\{\tau_i\}_{i=1}^K$ each arm is played over the entire horizon of length $T$. This lemma is the cornerstone of the regret analysis in subsequent sections. The proof of Lemma~\ref{RegEx} employs some techniques used in expanding the variance term in Appendix~A of~\cite{Sani}. However, we point out that~\cite{Sani} did not provide an exact expression of $\widehat{R}_\pi(T)$; rather, the results were obtained using an approximate of $\widehat{R}_\pi(T)$ (referred to as pseudo-regret in~\cite{Sani}).

\begin{lemma}\label{RegEx}
The regret of a policy $\pi$ with respect to the optimal single-arm policy $\widehat{\pi}^*$ under the measure of mean-variance of observations can be written as
\begin{eqnarray}
\widehat{R}_\pi(T) &=& \sum_{i=1}^K\mathbb{E}[\tau_i(T)]\Delta_i+\sum_{i=1}^K\mathbb{E}[\tau_i(T)]\Gamma_{i,*}^2-\frac{1}{T}\mathbb{E}[(\sum_{i=1}^K\tau_i(T)(\overline{\mu}_i(T)-\mu_*))^2]+\sigma_*^2.\label{RegExp}
\end{eqnarray}
\end{lemma}
\noindent{\it Proof:} See Appendix B.

Recall that the regret in terms of the total expected reward can be written as a weighted sum of the expected value of $\tau_i(T)$. Specifically, based on Wald identity, the regret is given by
\[
\sum_{i=1}^K \mathbb{E}[\tau_i(T)] (\mu_{\max}-\mu_i).
\]
The regret in terms of mean-variance of observations is, however, a much more complex function of $\tau_i(T)$ as given in Lemma~\ref{RegEx}. It depends on not only the expected value of $\tau_i(T)$, but also the second moment of $\tau_i(T)$ and the cross correlation between $\tau_i(T)$ and $\tau_j(T)$.

Based on Lemma~\ref{RegEx}, we show in Theorem~\ref{TheKtarin} that for $\Delta > 0$ and $T$ sufficiently large, the difference between $R_\pi(T)$ and
$\widehat{R}_{\pi}(T)$ is bounded by a constant independent of $T$.

\begin{theorem}\label{TheKtarin}
For any policy $\pi$, we have
\begin{eqnarray}\label{Alandige}
0\le R_\pi(T)-\widehat{R}_\pi(T)\le\min\{\sigma_{\max}^2(\sum_{i\neq*}\frac{\Gamma_{i,*}^2}{\Delta_i}+1),\frac{K}{a}\log T\}.
\end{eqnarray}

\end{theorem}

\noindent{\it Proof:} Since the performance of the optimal policy cannot be worse than the optimal single-arm policy, we can immediately see that $\widehat{R}_\pi(T)\le R_\pi(T)$. For the upper bound, we write $R_\pi(T)-\widehat{R}_\pi(T)=-\widehat{R}_{\pi^*}(T)$ and use the regret expression given in Lemma~\ref{RegEx} to establish lower bounds on $\widehat{R}_{\pi^*}(T)$. We first show that for $\Delta>0$ and large $T$, $\widehat{R}_{\pi^*}(T)$ is lower bounded by a constant. For the cases with small $\Delta$, we show that, based on Lemma~\ref{XS1} (proved in Appendix C), the difference between the second and the third terms on the RHS of~\eqref{RegExp} is bounded by an order of $\log T$ term. For a detailed proof, see Appendix D.

\begin{lemma}\label{XS1}
Let $\{X(t)\}_{t=1}^T$ be an i.i.d. random process with mean $\mu=\mathbb{E}[X(t)]$ that satisfies Lemma~\ref{LemmaChernoff} with constant $a$. Let $\tau\le T$ be an stopping time for this random process and let $\overline{\mu}$ denote the sample mean from $\tau$ samples: $\overline{\mu}=\frac{\sum_{s=1}^\tau X(s)}{\tau}$. We have the following inequality
\begin{eqnarray}\label{Lemma32}
\mathbb{E}[\tau(\overline{\mu}-\mu)^2]&\le& \frac{1}{a}(\log T+2).
\end{eqnarray}

\end{lemma}

\section{Model-Specific Regret}\label{SecFour}

In this section, we consider the model-specific setting. We establish lower bounds on model-specific regret feasible among all consistent policies and the order optimality of MV-UCB and MV-DSEE.

\subsection{Lower Bounds on Model-Specific Regret}

To avoid trivial lower bounds on regret caused by policies that heavily bias toward certain distribution models (e.g., a policy that always plays arm 1), the model-specific setting focuses on the so-called consistent policies.
The model-specific lower bounds for risk-neutral MAB (Theorems 1 and 2 in~\cite{Lai&Robbins85AAM}) are given for the set of policies that play suboptimal arms only $o(T^\alpha)$ times for all $\alpha\in (0,1)$.
We relax this assumption and focus on $\alpha-$consistent policies defined as follows.

\emph{Definition.} A policy $\pi$ is $\alpha$-consistent ($0<\alpha<1$) if for all reward distributions and for all $j\neq*$,
\begin{eqnarray}\label{alphacd}
\mathbb{E}[\tau_j(T)]\le T^\alpha.
\end{eqnarray}

We establish a lower bound on the model-specific regret feasible among the class of $\alpha$-consistent policies for all $\alpha\in(0,1)$. Similar to the results by Lai and Robbins in~\cite{Lai&Robbins85AAM} for risk-neutral MAB, we consider the family of one-parameter distribution models. Specifically, we assume that the distribution of arm $i$ is given by $f(.;\theta_i)$ and the distribution model $\mathcal{F}=(f(.;\theta_1),...,f(.;\theta_K))$ can be represented by ${\Theta}=(\theta_1,...,\theta_K)$. The parameters $\theta_i$ are taking value from a set $\Uc$ satisfying the following regularity condition (similar to that in~\cite{Lai&Robbins85AAM}).

\emph{Assumption 1.}
For any $\theta$, $\lambda$, and $\lambda'\in \mathcal{U}$, and for any $\epsilon>0$, there exists a $\delta>0$ such that $0<\xi(\lambda')-\xi(\lambda)<\delta$ implies $|I(f(.;\theta),f(.;\lambda))-I(f(.;\theta),f(.;\lambda'))|< \epsilon$.

The lower bound in~\cite{Lai&Robbins85AAM} is asymptotic ($T\rightarrow\infty$). In addition to establishing the corresponding asymptotic lower bound for risk-averse MAB, we also provide in Theorem~\ref{MSL} a
finite-time lower bound when the following assumption holds.

\emph{Assumption 2.}
For all $\theta$ and $\lambda\in \mathcal{U}$, let $X$ be a sub-Gaussian random variable with distribution $f(.;\theta)$. The random variable $Y=f(X;\lambda)$ is sub-Gaussian\footnote{Note that $Y$ is a function of $X$: for each $X=x$ generated according to $f(x;\theta)$, we have $Y=f(x;\lambda)$.}.

\begin{theorem}\label{MSL}
Consider the MAB problem under the measure of mean-variance of observations. Let $\pi$ be an $\alpha-$consistent policy and $\Theta\subset \mathcal{U}$ be the distribution model. Under Assumption 1, the model-specific regret satisfies, for any constant $c_1<1-\alpha$,
\begin{eqnarray}
\lim\inf_{T\rightarrow\infty}\frac{R_\pi(T)}{\log T} &\ge& \sum_{\substack{i=1\\i\neq*}}^{K} \frac{c_1}{I(f_i,f_*)}(\Delta_i+\Gamma_{i,*}^2).
\end{eqnarray}
Furthermore, under Assumption~2, for $T_1\in \mathbb{N}$,
\begin{eqnarray}
R_\pi(T) &\ge& \sum_{\substack{i=1\\i\neq*}}^{K} \frac{c_1c_2\log T }{I(f_i,f_*)}(\Delta_i+\Gamma_{i,*}^2-\epsilon_{T_1}),~~~~~\textit{for all}~~ T>T_1,
\end{eqnarray}
where $\epsilon_{T_1}$ can be arbitrary small when $T_1$ is large enough and $0<c_2<1$ is independent of $T$ and $\mathcal{F}$.
\end{theorem}
\begin{proof}
The proof is based on the following lemma.
\begin{lemma}~\label{LowMSPr}
Let $\Theta$ be the given distribution model, and let $i\neq *$ denote the index of a suboptimal arm under $\Theta$. Let $\pi$ be an $\alpha-$consistent policy. Under Assumption~1, the number $\tau_i(T)$ of times arm $i$ is played under $\pi$ satisfies, for any constant $c_1<1-\alpha$,
\begin{eqnarray}
\lim_{T\rightarrow\infty}\mathbb{P}_{\mathcal{F}}[\tau_i(T)\ge \frac{c_1\log T}{I(f_i,f_*)}] = 1,
\end{eqnarray}
Furthermore, under Assumption~2, there exists $ T_0\in \mathbb{N}$ such that

\begin{eqnarray}
\mathbb{P}_{\mathcal{F}}[\tau_i(T)\ge \frac{c_1\log T}{I(f_i,f_*)}]\ge c_2,~~~~~\textit{for all}~~ T>T_0,
\end{eqnarray}
where constant $0<c_2<1$ is independent of $T$ and $\mathcal{F}$.
\end{lemma}

To prove this lemma, we construct a new distribution model $\mathcal{F}^{i}$ where arm $i\neq*$ is the optimal arm. The log likelihood ratio $\gamma$ between the two probability measures $\mathcal{F}$ and $\mathcal{F}^{i}$ is a key statistic to prove the lemma. Specifically, we show that it is unlikely that $\tau_i$ is smaller than the logarithmic term under two different cases of $\gamma\le c_5\log T$ and $\gamma>c_5\log T$. The former is shown by a change of measure argument and using the consistency assumption. The latter is shown by Chernoff bound when Assumption 2 is satisfied and by law of large numbers otherwise. For a detailed proof see Appendix~E.

To prove Theorem~\ref{MSL}, we establish lower bounds on the first three terms of regret given in Lemma~\ref{RegEx}. Lemma~\ref{LowMSPr} provides a lower bound on $\mathbb{E}[\tau_i]$. By showing a lower bound on the sum of the second and third terms we arrive at the theorem. For a detailed proof see Appendix F.
\end{proof}

In comparison with Lai and Robbins lower bound for risk-neutral MAB~\cite{Lai&Robbins85AAM}, Theorem~\ref{MSL} considers a larger class of policies (by allowing a policy to be consistent with respect to a specific $\alpha$ rather than for all $\alpha\in (0,1)$) and also provides a finite-time lower bound under Assumption 2. Note that the constant $c_1$ in Theorem~\ref{MSL} approaches one for policies that satisfy~\eqref{alphacd} for all $\alpha\in (0,1)$, leading to a bound corresponding to that in~\cite{Lai&Robbins85AAM}.

\subsection{Risk-Averse Learning Policies}\label{SecFourr}
The performance of MV-UCB was first analyzed in~\cite{Sani}, which showed that the model-specific regret of MV-UCB was upper bounded by $O(\sqrt{T})$. Theorem~\ref{T2} below gives a tighter analysis on the performance of MV-UCB, showing a $\log T$ regret order.
This result, together with the lower bound given in Theorem~\ref{MSL}, establishes the order optimality of the MV-UCB policy for the case of $\Delta>0$.

MV-UCB assigns an index $\eta(t)$ to each arm and plays the arm with the smallest index at time~$t$ (after playing every arm once). The index depends on the sample mean-variance calculated from past observations and the number of times that the arm has been played up to time $t$. Specifically, the index of arm~$i$ at time~$t$ is given by\footnote{The index and the analysis of the MV-UCB can be modified for locally sub-Gaussian distributions following similar lines as in~\cite{UCBEx}.}
\begin{eqnarray}
\eta_i(t)= \overline{\xi}_i(t)-b\sqrt{\frac{\log t}{\tau_i (t)}},
\end{eqnarray}
where $b$ is a policy parameter whose value depends on the risk measure (see Theorem~\ref{T2} below).
%\begin{figure}[htbp]
%\begin{center}
%\noindent\fbox{
%\parbox{6in}
%{ \centerline{\underline{{\bf MV-UCB Policy}}} {
%\begin{itemize}
%\item Let $I_i(t)= \overline{\xi}_i(t)-b\sqrt{\frac{\log t}{\tau_i (t)}}$.
%\begin{enumerate}
%\item[1.] For $t=1,...,K$, play the arm $\pi(t)=t$.
%\item[2.] For $t>K$, play the arm $\pi(t)=\arg\min_{i} I_i(t)$.
%\end{enumerate}
%\end{itemize}} }} \caption{The MV-UCB policy for the risk-averse MAB problem.}~\label{UCBFig}
%\end{center}
%\end{figure}

\begin{theorem}\label{T2}
Assume $\Delta>0$. The regret offered by the MV-UCB policy with $b\ge \frac{\sqrt3(2+\rho)}{\sqrt a}$ under the measure of mean-variance of observations is upper bounded by\begin{eqnarray}\nn
R_{MV-UCB}(T) &\le& \sum_{\neq*}(\frac{4b^2 \log T}{\min\{\Delta_i^2,4(2+\rho)^2\}} + 5)(\Delta_i+\Gamma_{i,*}^2)+\sigma_*^2\\
&&~~~+\min\{\sigma_{\max}^2(\sum_{i\neq*}\frac{\Gamma_{i,*}^2}{\Delta_i}+1),\frac{K}{a}\log T\}.
\end{eqnarray}
\end{theorem}
\begin{proof}
From the regret expression given in~\eqref{RegExp}, we need to first bound $\mathbb{E}[\tau_i]$ for $i\neq*$. This is established in the following lemma with proof given in Appendix~{G}.
\begin{lemma}\label{tauL}
Set $b\ge \frac{\sqrt3(2+\rho)}{\sqrt a}$. The expected number of times a sub-optimal arm $i\neq*$ with $\Delta_i>0$ is played is upper bounded by
\begin{eqnarray}
\mathbb{E}[\tau_i(T)]\le \frac{4b^2 \log T}{\min\{\Delta_i^2,4(2+\rho)^2\}} + 5. \label{ubtu}
\end{eqnarray}
\end{lemma}
The third term in the regret expression in~\eqref{RegExp} is negative. Thus, we arrive at an upper bound on $\widehat{R}_{MV-UCB}(T)$ that translates to an upper bound on ${R}_{MV-UCB}(T)$ by applying Theorem~\ref{TheKtarin}.
See Appendix~H for a detailed proof.
\end{proof}

The model-specific regret of MV-UCB is linear in $T$ when $\Delta=0$ as discussed in~\cite{Sani}. An alternative policy in this case is MV-DSEE, a variation of the DSEE policy developed in~\cite{DSEE} for risk-neutral MAB. In the MV-DSEE policy, time is partitioned into two interleaving sequences: an
exploration sequence denoted by $\mathcal{E}(t)$ and an exploitation sequence. In the former, the player plays all arms in
a round-robin fashion. In the latter, the player plays the arm with the smallest sample mean-variance.

With the cardinality of the exploration sequence set to $\lceil f(T)\log T\rceil$ where $f(.)$ is a positive increasing diverging sequence with an arbitrarily slow rate, MV-DSEE offers an asymptotic regret order of $O(f(T)\log T)$ (which can be arbitrarily close to the optimal logarithmic order) over a fixed distribution model without the assumption of $\Delta>0$. %Note that this result was not reported in~\cite{Sani}, which analyzed the performance of %MV-DSEE only under the model-independent setting.

\begin{theorem}
The regret of MV-DSEE policy under the measure of mean-variance of observations is upper bounded by
\begin{eqnarray}
R_{MV-DSEE}(T)=O(f(T)\log T),
\end{eqnarray}
where $f(T)$ is a positive increasing diverging sequence with an arbitrarily slow rate.
\end{theorem}

\begin{proof}
Following similar steps as in the performance analysis of DSEE given in~\cite{DSEE}, we can show that for $i\neq*$,
\begin{eqnarray}\label{vaghannanato}
\mathbb{E}[\tau_i]=O(f(T)\log T).
\end{eqnarray}
Also similar to the proof of Theorem~\ref{T2}, we have
\begin{eqnarray}\nn
R_{MV-DSEE}(T)&\le&\sum_{i\neq*} \mathbb{E}[\tau_i(T)] (\Delta_i + \Gamma_{i,*}^2)+\min\{\sigma_{\max}^2(\sum_{i\neq*}\frac{\Gamma_{i,*}^2}{\Delta_i}+1),\frac{K}{a}\log T\}.
\end{eqnarray}
By substituting the bound on $\mathbb{E}[\tau_i]$ given in~\eqref{vaghannanato}, we arrive at the theorem.

\end{proof}

\section{Model-Independent Regret}\label{SecFive}
In this section, we consider the model-independent setting, in which the performance of a policy is measured against the worst-case reward model specific to the policy and the horizon length $T$.
Specifically, let $R_\pi(T;\Fc)$ denote the expected total performance loss of policy $\pi$ over a horizon of length $T$ for a reward model $\Fc$. The model-independent regret is given by, for each $T$,
\begin{eqnarray}
R_\pi(T)=\sup_{\mathcal{F}}R_\pi(T;\mathcal{F}),
\end{eqnarray}
and we are interested in the order (in terms of $T$) of a thus defined $R_\pi(T)$.
%Regret is a metric within the mini-max framework. In other words, the regret of a learning policy is given by its worst-case performance among all possible system models, or more specifically, all possible arm reward distributions $\Fc$.
It is easy to see that for any MAB
problem, the model-independent regret order cannot be lower than
the model-specific regret order.

We establish an $\Omega(T^{2/3})$ lower bound on the model-independent regret of any policy.
Specifically, in the following theorem we show that there is distribution model such that the regret grows with $\Omega(T^{2/3})$.

\begin{theorem}~\label{Low1}
Consider the MAB problem under the measure of mean-variance of observations. The model-independent regret of any policy $\pi$ satisfies, for some constants $c_3>0$ and $T_2\in \mathbb{N}$,
\begin{eqnarray}\label{20fin}
R_\pi(T) \ge c_3 T^{2/3},~~~~~\textit{for all}~~ T>T_2.
\end{eqnarray}
\end{theorem}

\begin{proof}
The proof is based on a coupling argument between two bandit problems with $K=2$ and under distribution models $\mathcal{F}$ and $\mathcal{F}'$, respectively. The optimal arm is switched under these two models while the difference $\Delta$ between the mean-variances of the optimal and the suboptimal arm is kept the same. First, it is shown that under at least one of these two distribution models, for some constants $c_4>0$ and $T_2\in \mathbb{N}$,
\begin{eqnarray}
R_\pi(T)\ge \frac{c_4\log T}{\Delta^2},~~~~~\textit{for all}~~ T>T_2\label{20m}.
\end{eqnarray}

Under both $\mathcal{F}$ and $\mathcal{F}'$, a normal distribution is assigned to arm one. Two different Bernolli distributions are assigned to arm two such that arm two is the sub-optimal arm under $\mathcal{F}$ and the optimal arm under $\mathcal{F}'$. Through a coupling argument we show that for the specific distribution assignments designed here,
\begin{eqnarray}\label{24jad}
\mathbb{P}_\mathcal{F}[\pi(t)=2]+\mathbb{P}_{\mathcal{F}'}[\pi(t)=1]\ge \exp(-\mathbb{E}_\mathcal{F} [\tau_2(T)] d_0 \Delta^2)
\end{eqnarray}
for some constant $d_0>0$.
A lower bound on regret can be derived from~\eqref{24jad}, which increases as $\mathbb{E}_\mathcal{F} [\tau_2(T)]$ decreases. On the other hand, a higher $\mathbb{E}_\mathcal{F} [\tau_2(T)]$ indicates a higher regret under distribution assignment $\mathcal{F}$. We show that, for any value of $\mathbb{E}_\mathcal{F} [\tau_2(T)]$, the maximum of these two lower bounds on regret is no smaller than the desired lower bound given in{~\eqref{20m}}.
A proper assignment of $\Delta=d_6T^{-\frac{1}{3}}$, for some constant $d_6$, gives the lower bound on model-independent regret in{~\eqref{20fin}}. For a detailed proof, see {Appendix~I}.
\end{proof}

MV-DSEE policy was also considered in~\cite{Sani} and was shown to achieve $O(T^{2/3})$ model-independent regret performance with the cardinality of the exploration sequence set to $|\mathcal{E}(T)|=\lceil T^{2/3}\rceil$.
The lower bound given in Theorem~\ref{Low1} shows that MV-DSEE is order optimal under the model-independent setting.

\section{Simulations}\label{SecSix}
In this section, we provide numerical examples on the performance of MV-UCB. We first study the effect of risk tolerance factor $\rho$ on the rewards obtained by a risk-averse policy. In Fig.~\ref{Sim1} two sample returns of MV-UCB are shown. By decreasing $\rho$ the variation in the observation decreases, although it is at a price of a lower average return.

\begin{figure}[htb]
\centering
  \begin{tabular}{@{}cccc@{}}

  \psfrag{r}[c]{\scriptsize{$\rho=1$}}
  %\psfrag{t}[c]{\scriptsize{$t$}}
\scalefig{0.50}\epsfbox{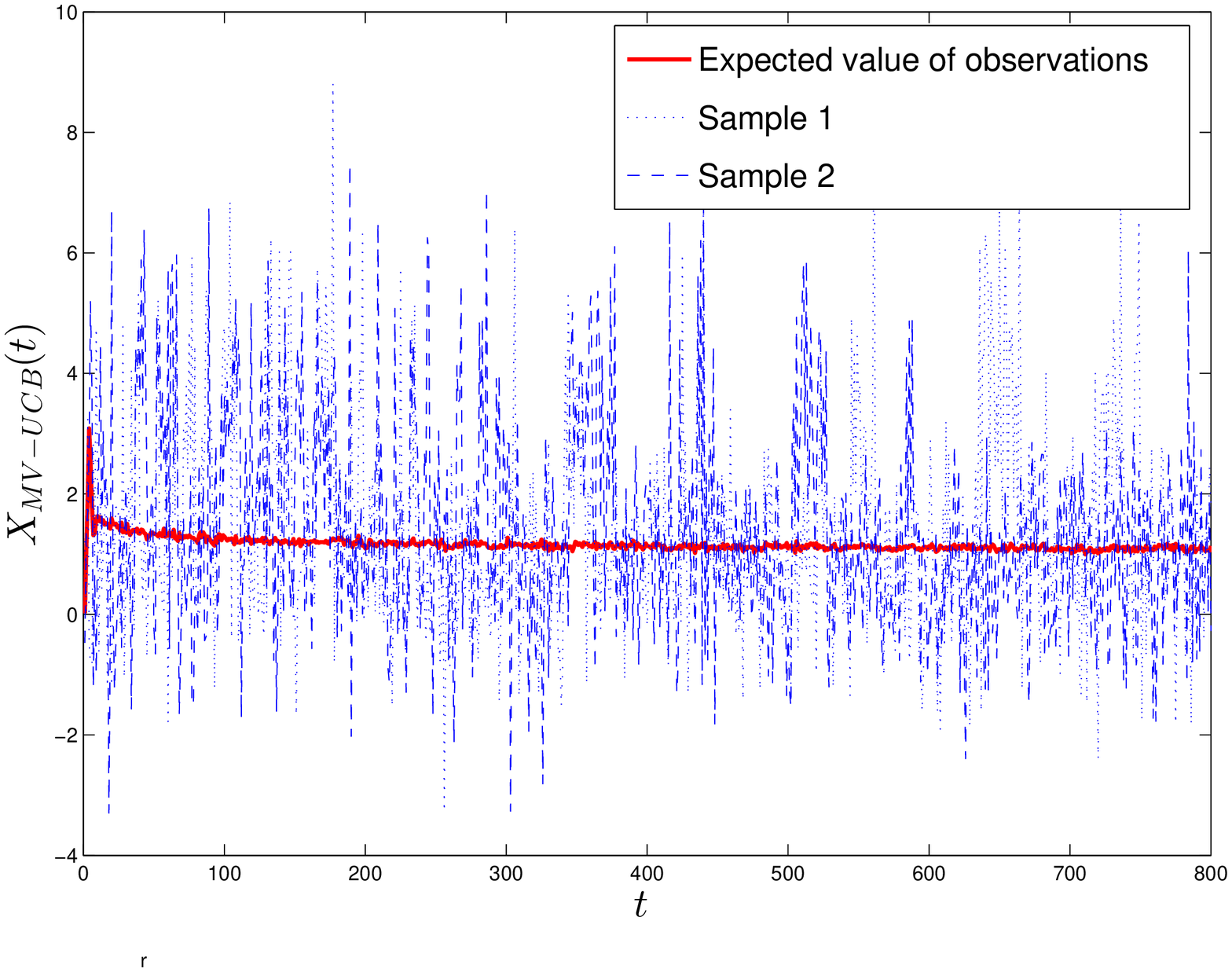}&
    \psfrag{r}[c]{\scriptsize{$\rho=5$}}
   % \psfrag{t}[c]{\scriptsize{$t$}}
\scalefig{0.50}\epsfbox{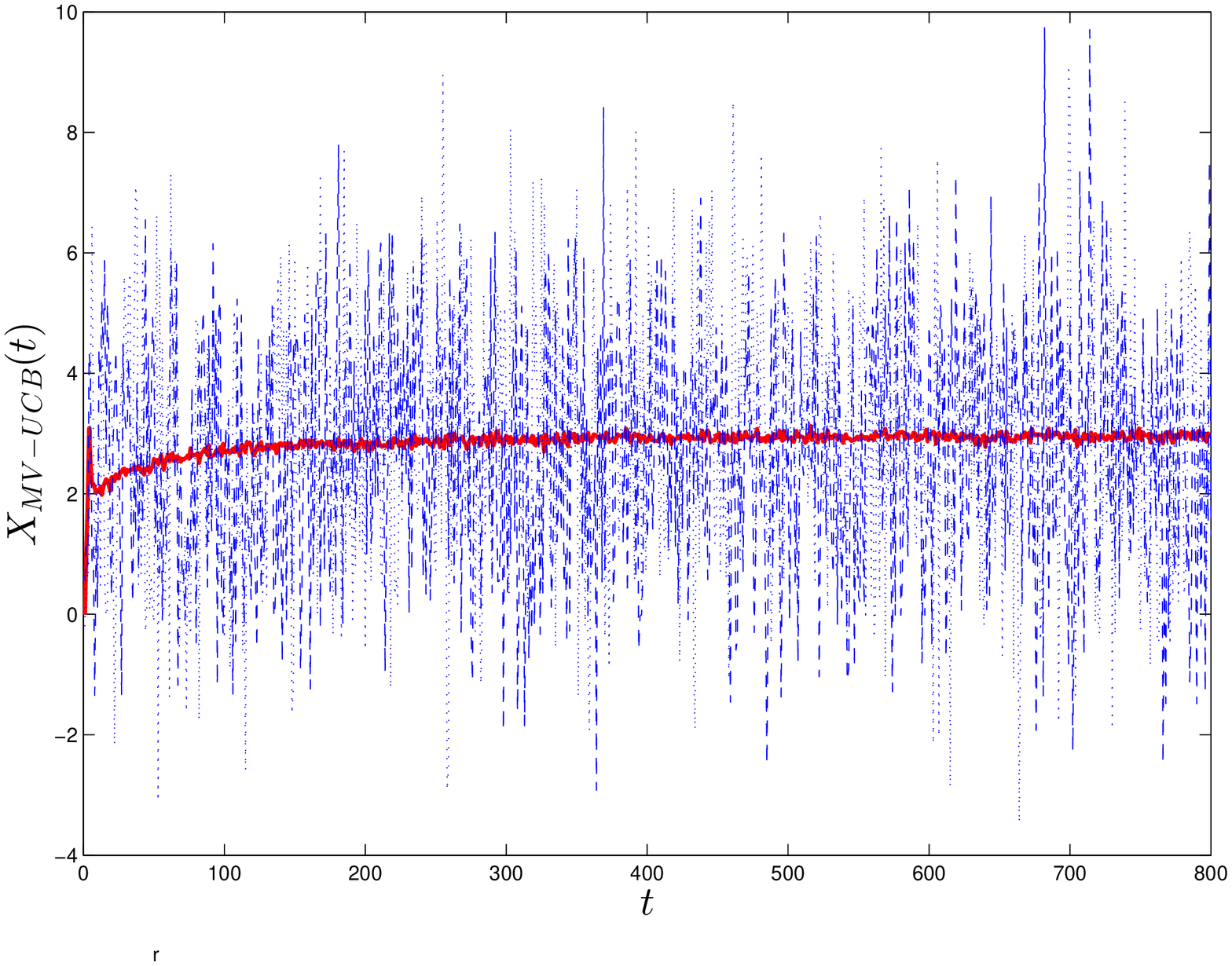}&\\
 %\psfrag{r}[c]{\scriptsize{$\rho=1$}}
%  \psfrag{t}[c]{\scriptsize{$t$}}
%\scalefig{0.50}\epsfbox{Simulation/Observations/ro1.eps}&
%    %\includegraphics[width=.50\textwidth]{Simulation/Observations/ro2} &
%    \psfrag{r}[c]{\scriptsize{$\rho=0.8$}}
%    \psfrag{t}[c]{\scriptsize{$t$}}
%\scalefig{0.50}\epsfbox{Simulation/Observations/ro08.eps}&\\
%\psfrag{r}[c]{\scriptsize{$\rho=0.5$}}
%  \psfrag{t}[c]{\scriptsize{$t$}}
%\scalefig{0.50}\epsfbox{Simulation/Observations/ro05.eps}&
% \psfrag{r}[c]{\scriptsize{$\rho=0$}}
%  \psfrag{t}[c]{\scriptsize{$t$}}
%\scalefig{0.50}\epsfbox{Simulation/Observations/ro0.eps}&

  \end{tabular}
  \caption{The sample observations of MV-UCB under different risk-tolerance factor $\rho$ ($K=4$, with normal reward distributions of parameters $\mu_1=0$, $\mu_2=1$, $\mu_3=2$, $\mu_4=3$, $\sigma_1=1$, $\sigma_2=1$, $\sigma_3=2$, $\sigma_4=2$).}\label{Sim1}
\end{figure}

Fig.~\ref{deltas} shows the regret performance of MV-UCB for different values of $\Delta$. The simulation shows that for a fixed value of $\Gamma$, the regret offered by MV-UCB increases as $\Delta$ decreases. A linear regret order is expected as $\Delta$ approaches $0$.

%Fig.~\ref{Sim5} compares the regret performance of the MV-UCB and MV-DSEE policies for a case with small $\Delta$. While regret performance of the MV-UCB is linear with time, the MV-DSEE policy shows an $O(T^{2/3})$ regret performance. The experiment is done a 2-armed bandit with parameters: $\mu_1=0$, $\mu_2=0.1$, $\sigma_1^2=1$, $\sigma_2^2=1.99$, $\rho=1$ and $\Delta0.01=$.

\begin{figure}[htb]\label{bbb}
\centering
  \begin{tabular}{@{}cccc@{}}

\scalefig{0.6}\epsfbox{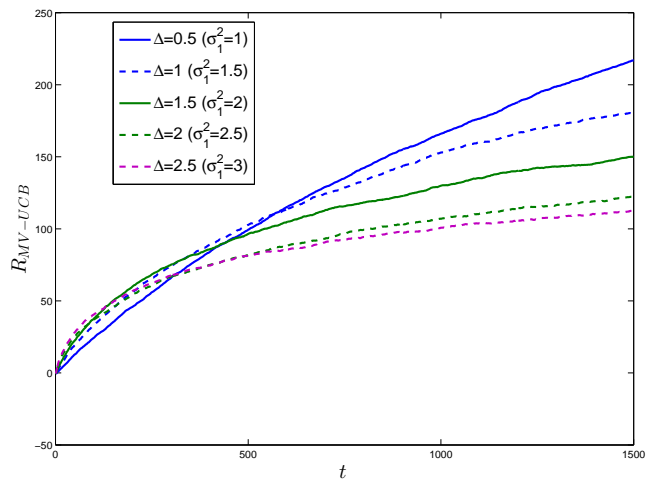}&\\
%\scalefig{0.90}\epsfbox{Simulation/Deltaeffect/R2.eps}&
  \end{tabular}
  \caption{The performance of MV-UCB ($\rho=1$, $K=2$ with normal reward distributions of parameters $\mu_1=0$, $\mu_2=0.5$, $\sigma_2^2=1$).} \label{deltas}
\end{figure}

\section{Conclusion and Discussion}

\par We studied risk-averse MAB problems under the risk measure of mean-variance of observations. We fully characterized the regret growth rate in both the model-specific and the model-independent settings by establishing lower bounds and developing order-optimal online learning policies.

The risk-averse MAB model reduces to the classic risk-neutral MAB when $\rho\rightarrow\infty$. Specifically, when $\rho\rightarrow \infty$, the mean-variance approaches to the negative of the mean multiplied by $\rho$. Thus, the mean-variance measure degenerates to a scaled mean value measure. With $\Delta_i$ replaced by $-\rho\Gamma_{i,*}$ and $\Gamma_{i,*}^2$ negligible against the term $-\rho\Gamma_{i,*}$, the model-specific bounds given in Theorems~\ref{MSL} and~\ref{T2} reproduce the bounds on risk-neutral regret.
Regarding the model-independent regret, however, as it is shown in this paper, the regret growth rate is different from the risk-neutral MAB. This difference is expected due to the reason that the worst-case assignment of the distributions takes into account the value of $\rho$. Thus, even for a large value of $\rho$, a proper choice of the distributions with a sufficiently small difference $\Gamma_{i,*}$ in the mean values results in a case where the difference in variance is comparable with $-\rho\Gamma_{i,*}$ and cannot be ignored.

Another scenario in which the risk-averse MAB under mean-variance approaches the risk-neutral MAB is when all arms have the same mean (i.e., $\Gamma\rightarrow 0$). Specifically, the time variations in the reward process have two sources: the randomness of the observed reward from each arm and the switching across arms with different expected values. The latter diminishes when $\Gamma\rightarrow0$.
Consequently, when $\Gamma\rightarrow0$, the regret in mean-variance of observations becomes summable over time and is given by a weighted sum of the
expected number of times that each suboptimal arm is played with the weights given by the difference in the variance of a suboptimal arm from the optimal arm. It is thus similar to the risk-neutral regret with the difference in mean replaced by the difference in variance. Thus, as expected, the model-specific bounds given in Theorems~\ref{MSL} and~\ref{T2} degenerate to the bounds on risk-neutral regret, except that $\Delta_i$ is the difference in the variance rather than the mean.
Under the model-independent setting, the value of $\Gamma$ is chosen for the worst-case assignment of the distribution model and cannot be forced to zero. The above connection through $\Gamma\rightarrow 0$ between the regret in mean-variance and the regret in mean is thus absent in the model-independent setting.

The model-specific regret lower bound obtained in Theorem~\ref{MSL} applies to only single-parameter distribution models, the same assumption as used in the lower bound obtained by Lai and Robbins in \cite{Lai&Robbins85AAM} for risk-neutral MAB. Under the measure of mean-variance of observations,
the mean and the mean-variance of each arm are dependent through the single parameter $\theta_i$ of the distribution. Thus, the values of $\Delta_i$ and $\Gamma_{i,*}$ cannot be set independently. As a result, the $\Omega(T^{\frac{2}{3}})$ regret lower bound in the model-independent setting (where the worst-case values of $\Delta_i$ and $\Gamma_{i,*}$ can be chosen independently) cannot be deduced from Theorem~\ref{MSL}.
The regret performance of MV-UCB and MV-DSEE as given in Theorem~3 and Theorem~4, however, does not require the assumption of single-parameter distribution models. It is thus perhaps reasonable to expect that the logarithmic order in the lower bound holds for general distribution models.

Our regret lower bounds that hold for all $T\ge T_0$ for some constant $T_0\in \mathbb{N}$ should be interpreted as finite-time results since one can always find a leading constant large enough (in the case of upper bounds) or small enough (in the case of lower bounds) to accommodate the first $T_0$ terms. Indeed, how large or small the leading constant needs to be to have results hold for all $T$ can be obtained in our proof procedure. However, such a practice is tedious and leads to an overly complicated expression.

For the risk-neutral MAB, an improved version of the UCB policy developed in \cite{UCBRev} was shown to achieve the optimal regret order under both the model-specific and model-independent settings. We have shown in this paper that MV-DSEE approaches both the model-specific and model-independent regret lower bounds, but requiring different values for the cardinality of the exploration sequence.
Whether a single policy without any change in its parameter values can achieve the optimal regret order under both settings remains an open question. A satisfactory answer to this question is involved and requires a separate investigation.

\section*{Acknowledgment}
The authors would like to thank the anonymous reviewers for their invaluable suggestions and comments that have significantly improved both the presentation and the results of this paper.

\subsection*{Appendix A: Proof of Lemma 1 }\label{AppA}

Let $\overline{\mu}_s$ be the sample mean obtained from $s$ i.i.d. observations. By Chernoff-Hoeffding bound~\cite{Chernoff,UCBEx}, for all $a\in(0,\frac{1}{2\zeta_0}]$,

\begin{eqnarray}
\left\{
\begin{array}{ll}\nn
\mathbb{P}[\overline{\mu}_s-\mu(X)<-\delta_1]&\le  \exp(-{as\delta_1^2}),\\
\mathbb{P}[\overline{\mu}_s-\mu(X)>\delta_1]&\le  \exp(-{as\delta_1^2}),
\end{array}\right.
\end{eqnarray}
and, for all $a\in(0,\frac{1}{2\zeta_1}]$,
\begin{eqnarray}\nn
\left\{
\begin{array}{ll}
\mathbb{P}[|\frac{1}{s}\sum_{t=1}^{s}(X(t)-\mu(X))^2-\sigma^2(X)|<-\delta_2]&\le  \exp(-{as\delta_2^2}),\\
\mathbb{P}[|\frac{1}{s}\sum_{t=1}^{s}(X(t)-\mu(X))^2-\sigma^2(X)|>\delta_2]&\le \exp(-{as\delta_2^2}),
\end{array}\right.
\end{eqnarray}
where $X(t)$ is the $t$th observation of the random variable $X$.
The mean-variance deviation term can be written as
\begin{eqnarray}\nn
\overline{\xi}_s-\xi(X)&=& \frac{1}{s}\sum_{t=1}^{s}(X(t)-\overline{\mu}_s)^2-\rho\overline{\mu}_s-\xi(X)\\\nn
&=&\frac{1}{s}\sum_{t=1}^{s}(X(t)-\mu(X))^2+(\mu(X)-\overline{\mu}_s)^2\\\nn
&+&\frac{2}{s}\sum_{t=1}^{s}(X(t)-\mu(X))(\mu(X)-\overline{\mu}_s)-\rho\overline{\mu}_s-\xi(X)\\
&=&\frac{1}{s}\sum_{t=1}^{s}(X(t)-\mu(X))^2-\sigma^2(X) - (\mu(X)-\overline{\mu}_s)^2 -\rho(\overline{\mu}_s-\mu(X)).~~~~~~\label{43}
\end{eqnarray}
Notice that the second term on the right hand side of~\eqref{43} is always negative. For $\delta_1=\frac{\delta}{1+\rho}$ and $a\le\frac{1}{2\zeta}$, substituting $\overline{\xi}_s-\xi(X)$ from~\eqref{43}
\begin{eqnarray}\nn
\mathbb{P}[\overline{\xi}_s-\xi(X)>\delta] &\le&  \mathbb{P}[\frac{1}{s}\sum_{t=1}^{s}(X(t)-\mu(X))^2-\sigma^2(X) >\delta_1] + \mathbb{P}[\overline{\mu}_s-\mu(X)<-\delta_1]\\\nn
&\le& \exp(-{as\delta_1^2}) + \exp(-{as\delta_1^2}) \\
&=&2 \exp(-\frac{as\delta^2}{(1+\rho)^2}).
\end{eqnarray}

To prove~\eqref{narnia2} let $\delta_1=\frac{\delta}{2+\rho}$. Notice that, $\delta_1\le 1$ when $\delta\le2+\rho$ and $(\mu(X)-\overline{\mu}_s)<\delta_1$ implies $(\mu(X)-\overline{\mu}_s)^2<\delta_1$.
For $a\le\frac{1}{2\zeta}$, substituting $\overline{\xi}_s-\xi(X)$ from~\eqref{43}
\begin{eqnarray}\nn
\mathbb{P}[\overline{\xi}_s-\xi(X)<-\delta] &\le&  \mathbb{P}[\frac{1}{s}\sum_{t=1}^{s}(X(t)-\mu(X))^2-\sigma^2(X) < -\delta_1] + \mathbb{P}[\overline{\mu}_s-\mu(X)>\delta_1]\\\nn
&\le& \exp(-{as\delta_1^2}) + \exp(-{as\delta_1^2}) \\
&=&2 \exp(-\frac{as\delta^2}{(2+\rho)^2}).
\end{eqnarray}

\subsection*{Appendix B: Proof of Lemma 2 }

Let $\overline{\mu}_\pi=\frac{1}{T}\sum_{t=1}^T X_{\pi(t)}(t)$ and $\mu_\pi=\mathbb{E}[\overline{\mu}_\pi]$. In order to show the expression of regret given in~\eqref{RegEx}, we expand the cumulative variance term.
\begin{eqnarray}\nn
\mathbb{E}[\sum_{t=1}^T (X_{\pi(t)}(t)-\overline{\mu}_\pi)^2]
&=&  \mathbb{E}[\sum_{i=1}^K\sum_{s=1}^{\tau_i(T)}(X_i(t_i(s))-\overline{\mu}_\pi)^2]\\\nn
&=&  \mathbb{E}[\sum_{i=1}^K\sum_{s=1}^{\tau_i(T)}(X_i(t_i(s))-{\mu}_i+{\mu}_i-\overline{\mu}_\pi)^2]\\\nn
&=&  \mathbb{E}[\sum_{i=1}^K\sum_{s=1}^{\tau_i(T)}((X_i(t_i(s))-{\mu}_i)^2+ ({\mu}_i-\overline{\mu}_\pi)^2\\\label{rv21}
&&~~~+2(X_i(s)-{\mu}_i)({\mu}_i-\overline{\mu}_\pi) )].
\end{eqnarray}
The first term on the RHS of~\eqref{rv21} equals to, by Wald identity,
\begin{eqnarray}\label{rv22}
\mathbb{E}[\sum_{i=1}^K\sum_{s=1}^{\tau_i(T)}((X_i(s)-{\mu}_i)^2]=\sum_{i=1}^K\mathbb{E}\tau_i\sigma_i^2.
\end{eqnarray}
The second term can be written as
\begin{eqnarray}\nn
\mathbb{E}[\sum_{i=1}^K\tau_i({\mu}_i-\overline{\mu}_\pi)^2]&=&\mathbb{E}[\sum_{i=1}^K\tau_i({\mu}_i-\mu_\pi+\mu_\pi-\overline{\mu}_\pi)^2]\\\nn
&=&\mathbb{E}[\sum_{i=1}^K\tau_i(({\mu}_i-\mu_\pi)^2+(\mu_\pi-\overline{\mu}_\pi)^2+2({\mu}_i-\mu_\pi)(\mu_\pi-\overline{\mu}_\pi))]\\\nn
&=&\sum_{i=1}^K\mathbb{E}[\tau_i({\mu}_i-\mu_\pi)^2]+\mathbb{E}[T(\mu_\pi-\overline{\mu}_\pi)^2]\\\label{rv23}
&&~~~+2\mathbb{E}[\sum_{i=1}^K\tau_i({\mu}_i-\mu_\pi)(\mu_\pi-\overline{\mu}_\pi)].
\end{eqnarray}
The third term can be written as
\begin{eqnarray}\nn
\mathbb{E}[\sum_{i=1}^K2\tau_i(\overline{\mu}_i-\mu_i)(\mu_i-\overline{\mu}_\pi)]&=&2\mathbb{E}[\sum_{i=1}^K\tau_i(\overline{\mu}_i-\mu_i)(\mu_i-\mu_\pi)+\tau_i(\overline{\mu}_i-\mu_i)(\mu_\pi-\overline{\mu}_\pi)]\\\label{rv24}
&=&2\mathbb{E}[\sum_{i=1}^K\tau_i(\overline{\mu}_i-\mu_i)(\mu_\pi-\overline{\mu}_\pi)].
\end{eqnarray}

From~\eqref{rv22},~\eqref{rv23} and~\eqref{rv24}, we have

\begin{eqnarray}\nn
\mathbb{E}[\sum_{t=1}^T (X_{\pi(t)}(t)-\overline{\mu}_\pi)^2]&=&\sum_{i=1}^K\mathbb{E}\tau_i\sigma_i^2+\sum_{i=1}^K\mathbb{E}[\tau_i({\mu}_i-\mu_\pi)^2]+\mathbb{E}[T(\mu_\pi-\overline{\mu}_\pi)^2]\\\nn
&&~~~+2\mathbb{E}[\sum_{i=1}^K\tau_i({\mu}_i-\mu_\pi)(\mu_\pi-\overline{\mu}_\pi)]+2\mathbb{E}[\sum_{i=1}^K\tau_i(\overline{\mu}_i-\mu_i)(\mu_\pi-\overline{\mu}_\pi)]\\\nn
&=& \sum_{i=1}^K\mathbb{E}\tau_i\sigma_i^2+\sum_{i=1}^K\mathbb{E}[\tau_i({\mu}_i-\mu_\pi)^2]+\mathbb{E}[T(\mu_\pi-\overline{\mu}_\pi)^2]\\\nn
&&~~~+2\mathbb{E}[\sum_{i=1}^K\tau_i(\overline{\mu}_i-\mu_\pi)(\mu_\pi-\overline{\mu}_\pi)]\\\label{dum}
&=& \sum_{i=1}^K\mathbb{E}\tau_i\sigma_i^2+\sum_{i=1}^K\mathbb{E}[\tau_i({\mu}_i-\mu_\pi)^2]+\mathbb{E}[T(\mu_\pi-\overline{\mu}_\pi)^2]-2\mathbb{E}[T(\mu_\pi-\overline{\mu}_\pi)^2]~~~~~~\\\label{rv25}
&=& \sum_{i=1}^K\mathbb{E}\tau_i\sigma_i^2+\sum_{i=1}^K\mathbb{E}[\tau_i({\mu}_i-\mu_\pi)^2]-\mathbb{E}[T(\mu_\pi-\overline{\mu}_\pi)^2].
\end{eqnarray}
To arrive at~\eqref{dum}, $\sum_{i=1}^K\tau_i(T)=T$ and $\sum_{i=1}^K\tau_i(T)\overline{\mu}_i(T)=T\overline{\mu}_\pi$ are used. Similarly, $\sum_{i=1}^K\tau_i(T){\mu}_i=T{\mu}_\pi$ and it can be shown that

{{\begin{eqnarray}\nn
\sum_{i=1}^K\mathbb{E}[\tau_i({\mu}_i-\mu_\pi)^2]&=&\sum_{i=1}^K\mathbb{E}[\tau_i({\mu}_i-\frac{\sum_{j=1}^K\mathbb{E}\tau_j\mu_j}{T})^2]\\\nn
&=&\frac{1}{T^2}\sum_{i=1}^K\mathbb{E}[\tau_i(\sum_{j=1}^K\mathbb{E}\tau_j\Gamma_{i,j})^2]\\\nn
&=&\frac{1}{T^2}\mathbb{E}[\tau_*(\sum_{j=1}^K\mathbb{E}\tau_j\Gamma_{*,j})^2]+\frac{1}{T^2}\sum_{i\neq*}\mathbb{E}[\tau_i(\sum_{j\neq*}\mathbb{E}\tau_j\Gamma_{i,j}+\mathbb{E}\tau_*\Gamma_{i,*})^2]\\\nn
&=&\frac{1}{T^2}\mathbb{E}[(T-\sum_{i\neq*}\tau_i)(\sum_{j=1}^K\mathbb{E}\tau_j\Gamma_{*,j})^2]+\frac{1}{T^2}\sum_{i\neq*}\mathbb{E}[\tau_i(\sum_{j\neq*}\mathbb{E}\tau_j\Gamma_{i,j}+(T-\sum_{j\neq*}\mathbb{E}\tau_j)\Gamma_{i,*})^2]\\\nn
&=&\frac{1}{T}(\sum_{j=1}^K\mathbb{E}\tau_j\Gamma_{*,j})^2-\frac{1}{T^2}\sum_{i\neq*}\mathbb{E}[\tau_i(\sum_{j=1}^K\mathbb{E}\tau_j\Gamma_{*,j})^2]\\\nn
&&~~+\frac{1}{T^2}\sum_{i\neq*}\mathbb{E}[\tau_i(T\Gamma_{i,*}+\sum_{j\neq*}\mathbb{E}\tau_j\Gamma_{*,j})^2]\\\nn
&=&\frac{1}{T}(\sum_{j=1}^K\mathbb{E}\tau_j\Gamma_{*,j})^2-\frac{1}{T^2}\sum_{i\neq*}\mathbb{E}[\tau_i(\sum_{j=1}^K\mathbb{E}\tau_j\Gamma_{*,j})^2]\\\nn
&&~~+\sum_{i\neq*}\mathbb{E}\tau_i\Gamma_{i,*}^2+\frac{1}{T^2}\sum_{i\neq*}\mathbb{E}[\tau_i(\sum_{j\neq*}\mathbb{E}\tau_j\Gamma_{*,j})^2]+\frac{2}{T}\sum_{i\neq*}\mathbb{E}[\tau_i\Gamma_{i,*}(\sum_{j=1}^K\mathbb{E}\tau_j\Gamma_{*,j})]~~~\\\label{rv251}
&=&\sum_{i\neq*}\mathbb{E}\tau_i\Gamma_{i,*}^2-\frac{1}{T}(\sum_{j=1}^K\mathbb{E}\tau_j\Gamma_{*,j})^2.
\end{eqnarray}}}
For the third term on the RHS of~\eqref{rv25}, we have
{\begin{eqnarray}\nn
&&\hspace{-3em}\mathbb{E}[T(\mu_\pi-\overline{\mu}_\pi)^2]\\\nn
&=&\frac{1}{T}\mathbb{E}[(\sum_{i=1}^K\sum_{s=1}^{\tau_i}X_i(t_i(s))-\sum_{i=1}^K\mathbb{E}\tau_i\mu_i)^2]\\\nn
&=&\frac{1}{T}\mathbb{E}[(\sum_{i=1}^K\sum_{s=1}^{\tau_i}(X_i(t_i(s))-\mu_i)+\sum_{i=1}^K(\tau_i-\mathbb{E}\tau_i)\mu_i)^2]\\\nn
&=&\frac{1}{T}\mathbb{E}[(\sum_{i=1}^K\sum_{s=1}^{\tau_i}(X_i(t_i(s))-\mu_i))^2]+\frac{1}{T}\mathbb{E}[(\sum_{i=1}^K(\tau_i-\mathbb{E}\tau_i)\mu_i)^2]\\\nn
&&~~~+\frac{2}{T}\mathbb{E}[(\sum_{i=1}^K\sum_{s=1}^{\tau_i}(X_i(t_i(s))-\mu_i))(\sum_{i=1}^K(\tau_i-\mathbb{E}\tau_i)\mu_i)]\\\nn
&=&\frac{1}{T}\mathbb{E}[(\sum_{i=1}^K\tau_i(\overline{\mu}_i-\mu_i))^2]+\frac{1}{T}\mathbb{E}[(\sum_{i=1}^K(\tau_i-\mathbb{E}\tau_i)\Gamma_{i,*})^2]\\\label{dum2}
&&~~~+\frac{2}{T}\mathbb{E}[(\sum_{i=1}^K\tau_i(\overline{\mu}_i-\mu_i))(\sum_{i=1}^K(\tau_i-\mathbb{E}\tau_i)\Gamma_{i,*})]\\\nn
&=&\frac{1}{T}\mathbb{E}[(\sum_{i=1}^K\tau_i(\overline{\mu}_i-\mu_i))^2]+ \frac{1}{T}\mathbb{E}[(\sum_{i=1}^K\tau_i\Gamma_{i,*})^2]-\frac{1}{T}(\sum_{i=1}^K\mathbb{E}\tau_i\Gamma_{i,*})^2\\\label{rv26}
&&~~~+\frac{2}{T}\mathbb{E}[(\sum_{i=1}^K\tau_i(\overline{\mu}_i-\mu_i))(\sum_{i=1}^K\tau_i\Gamma_{i,*})].
\end{eqnarray}}
Equation~\eqref{dum2} follows from  $\sum_{i=1}^K(\tau_i-\mathbb{E}\tau_i)\mu_i=(\tau_*-\mathbb{E}\tau_*)\mu_*+\sum_{i\neq*}(\tau_i-\mathbb{E}\tau_i)\mu_i=
-\sum_{i\neq*}(\tau_i-\mathbb{E}\tau_i)\mu_*+\sum_{i\neq*}(\tau_i-\mathbb{E}\tau_i)\mu_i=\sum_{i\neq*}(\tau_i-\mathbb{E}\tau_i)\Gamma_{i,*}$. We know that for any random variable $X$, $\sigma^2(X)=\mathbb{E}[X^2]-\mathbb{E}[X]^2$. To arrive at~\eqref{rv26}, set $X=\sum_{i=1}^K\tau_i\Gamma_{i,*}$ also notice that $\frac{2}{T}\mathbb{E}[(\sum_{i=1}^K\tau_i(\overline{\mu}_i-\mu_i)(\sum_{i=1}^K\mathbb{E}\tau_i\Gamma_{i,*})]=0$.

%-\frac{2}{T}\mathbb{E}[(\sum_{i=1}^K\tau_i(\overline{\mu}_i-\mu_i)(\sum_{i=1}^K\mathbb{E}\tau_i\Gamma_{i,*})]
Thus, from~\eqref{rv25},~\eqref{rv251} and~\eqref{rv26}, we have

\begin{eqnarray}\nn
\mathbb{E}[\sum_{t=1}^T (X_{\pi(t)}(t)-\overline{\mu}_\pi)^2]&=&\sum_{i=1}^K\mathbb{E}\tau_i\sigma_i^2+\sum_{i\neq*}\mathbb{E}\tau_i\Gamma_{i,*}^2-\frac{1}{T}\mathbb{E}[(\sum_{i=1}^K\tau_i(\overline{\mu}_i-\mu_i))^2]-\frac{1}{T}\mathbb{E}[(\sum_{i=1}^K\tau_i\Gamma_{i,*})^2]\\\nn
&&~~~-\frac{2}{T}\mathbb{E}[\sum_{i=1}^K\tau_i(\overline{\mu}_i-\mu_i)(\sum_{i=1}^K\tau_i\Gamma_{i,*})]\\\label{rr1}
&=&\sum_{i=1}^K\mathbb{E}\tau_i\sigma_i^2+\sum_{i\neq*}\mathbb{E}\tau_i\Gamma_{i,*}^2-\frac{1}{T}\mathbb{E}[(\sum_{i=1}^K\tau_i(\overline{\mu}_i-\mu_*))^2].
\end{eqnarray}

Now we can show the expression for $\widehat{R}_{\pi}(T)$ for any policy $\pi$ that plays arm $i$ for $\tau_i$ times
\begin{eqnarray}\nn
\widehat{R}_{\pi}(T)&=&\xi_{\pi}(T)-\xi_{\widehat{\pi}^*}(T)\\\nn
&=&\sum_{i=1}^K\mathbb{E}\tau_i\xi_i+\sum_{i\neq*}\mathbb{E}\tau_i\Gamma_{i,*}^2-\frac{1}{T}\mathbb{E}[(\sum_{i=1}^K\tau_i(\overline{\mu}_i-\mu_*))^2]-T\xi_*+\frac{1}{T}\mathbb{E}[T(\overline{\mu}_*-\mu_*)]\\\nn
&=&\sum_{i=1}^K\mathbb{E}\tau_i\Delta_i+\sum_{i=1}^K\mathbb{E}\tau_i\Gamma_{i,*}^2-\frac{1}{T}\mathbb{E}[(\sum_{i=1}^K\tau_i(\overline{\mu}_i-\mu_*))^2]+\sigma_*^2
\end{eqnarray}
as desired.

%%%%%%%%%%%%%%%%%%%%%%%%%%%%%%%%%%%%%%%%%%%%%%%%%%%%%%%%%%%%%%%%%%%%%%%%%%%%%%%%%%%%%%%%%%%%%%%%%%%%%%%%%%%%%%%%%%%%%%%%%%%%%%%%%%%%%%%%%%%%%%%%%%%%%%%%%%%%
%%%%%%%%%%%%
%%%%%%%%%%%%%
\subsection*{Appendix C: Proof of Lemma~\ref{XS1} }

In order to prove~\eqref{Lemma32}, we write the expected value of $\tau(\overline{\mu}-\mu)^2$ divided by $\log T$ as integrating the tail probability. For the tail probability we have, for a real number $x>0$\footnote{For the locally sub-Gaussian distributions, a similar bound on $\mathbb{P}[\tau(\overline{\mu}-\mu)^2>x\log^2 T]$ can be proven which results in an $O(\log^2T) $ term on the RHS of~\eqref{Lemma32}. Consequently, the second term on the RHS of~\eqref{Alandige} becomes $O(\log^2T)$, which does not affect the results on regret order developed in subsequent sections.}
\begin{eqnarray}\nn
\mathbb{P}[\tau(\overline{\mu}-\mu)^2>x\log T]&\le& \mathbb{P}[\max_{1\le s \le T}s(\overline{\mu}_{s}-\mu)^2>x\log T]\\\nn
&=& \mathbb{P}[\max_{1\le s \le T}\sqrt s|\overline{\mu}_{s}-\mu|>\sqrt{x\log T}]\\\nn
&\le&\sum_{s=1}^T  \mathbb{P}[|\overline{\mu}_{s}-\mu|>\sqrt{\frac{x\log T}{s}}]\\\nn
&\le& \sum_{s=1}^T2\exp(-a x\log T)\\\nn
&=&2T^{-ax+1}.
\end{eqnarray}
Now, we can write
\begin{eqnarray}\nn
\mathbb{E}[\frac{\tau (\overline{\mu}-\mu)^2}{\log T}]&=&\int_0^\infty \mathbb{P}[\frac{\tau (\overline{\mu}-\mu)^2}{\log T}>x]dx\\\nn
&\le& \frac{1}{a}+\int_{\frac{1}{a}}^\infty \mathbb{P}[\frac{\tau (\overline{\mu}-\mu)^2}{\log T}>x]dx\\\nn
&\le& \frac{1}{a}+\int_{\frac{1}{a}}^\infty2T^{-ax+1}dx\\\nn
&=&\frac{1}{a}+2\frac{T^{-ax+1}}{a\log T}|^\frac{1}{a}_\infty\\\nn
&=&\frac{1}{a}(1+\frac{2}{\log T}).
\end{eqnarray}
Thus, multiplying by $\log T$, we have%~\eqref{Lemma32}.
\begin{eqnarray}\nn
\mathbb{E}[\tau (\overline{\mu}-\mu)^2]\le \frac{1}{a}(\log T+ 2).
\end{eqnarray}

%%%%%%%%%%%%%%%%%%%%%%%%%%%%%%%%%%%%%%%%%%%%%%%%%%%%%%%%%%%%%%%%%%%%%%%%%%%%%%%%%%%%%%%%%%%%%%%%%%%%%%%%%%%%%%%%%%%%%%%%%%%%%%%%%%%%%%%%%%%%%%%%%%%%%%%%%%%%%%
\subsection*{Appendix D: Proof of Theorem~\ref{TheKtarin} }%\label{AppB}

Since $\xi_{\pi^*}\le \xi_{\widehat{\pi}^*}$, it is straightforward to see that
\begin{eqnarray}
R_{\pi}(T)-\widehat{R}_{\pi}(T)=\xi_{\pi}(T)-\xi_{\pi^*}(T)-(\xi_{\pi}(T)-\xi_{\widehat{\pi}^*}(T))\ge 0.
\end{eqnarray}

For the upper bound, we have
\begin{eqnarray}\nn
R_{\pi}(T)-\widehat{R}_{\pi}(T)&=&\xi_{\widehat{\pi}^*}(T)-\xi_{\pi^*}(T)\\\label{56en}
&=&-\widehat{R}_{\pi^*}(T).
\end{eqnarray}

From Lemma~\ref{RegEx}, we have

\begin{eqnarray}
\widehat{R}_{\pi^*}(T)&=&\sum_{i=1}^K\mathbb{E}\tau_i\Delta_i+\sum_{i=1}^K\mathbb{E}\tau_i\Gamma_{i,*}^2-\frac{1}{T}\mathbb{E}[(\sum_{i=1}^K\tau_i(\overline{\mu}_i-\mu_*))^2]+\sigma_*^2,
\end{eqnarray}
where $\tau_i$ are the number of times arm $i$ is played by $\pi^*$.
We have, by Cauchy-Schwartz inequality,
\begin{eqnarray}\nn
\frac{1}{T}\mathbb{E}[(\sum_{i=1}^K\tau_i(\overline{\mu}_i-\mu_*))^2]&=&\frac{1}{T}\mathbb{E}[(\sum_{i=1}^K\tau_i(\overline{\mu}_i-\mu_i))^2]+\frac{1}{T}\mathbb{E}[(\sum_{i=1}^K\tau_i\Gamma_{i,*})^2]\\\nn
&&~~~+\frac{2}{T}\mathbb{E}[\sum_{i=1}^K\tau_i(\overline{\mu}_i-\mu_i)(\sum_{i=1}^K\tau_i\Gamma_{i,*})]\\\nn
&\le&\frac{1}{T}\mathbb{E}[(\sum_{i=1}^K\tau_i(\overline{\mu}_i-\mu_i))^2]+\frac{1}{T}\mathbb{E}[(\sum_{i=1}^K\tau_i\Gamma_{i,*})^2]\\\nn
&&+\frac{2}{T}\sqrt{\mathbb{E}[(\sum_{i=1}^K\tau_i(\overline{\mu}_i-\mu_i))^2]\mathbb{E}[(\sum_{i=1}^K\tau_i\Gamma_{i,*})^2]}\\\nn
&=&\frac{1}{T}\sum_{i=1}^K\mathbb{E}\tau_i\sigma_i^2+\frac{1}{T}\mathbb{E}[(\sum_{i=1}^K\tau_i\Gamma_{i,*})^2]\\\label{Tue1}
&&~~~+\frac{2}{T}\sqrt{\sum_{i=1}^K\mathbb{E}\tau_i\sigma_i^2}\sqrt{\mathbb{E}[(\sum_{i=1}^K\tau_i\Gamma_{i,*})^2]}.
\end{eqnarray}
To arrive at~\eqref{Tue1}, we also use $\mathbb{E}[(\sum_{i=1}^K\tau_i(\overline{\mu}_i-\mu_i))^2]=\sum_{i=1}^K\mathbb{E}\tau_i\sigma_i^2$ as a result of Wald's second identity.
For the second term on the RHS we have, by applying again Cauchy-Schwartz inequality,
\begin{eqnarray}\nn
\frac{1}{T}\mathbb{E}[(\sum_{i=1}^K\tau_i\Gamma_{i,*})^2]&\le&\frac{1}{T}\mathbb{E}[(\sum_{i=1}^K\tau_i)(\sum_{i=1}^K\tau_i\Gamma_{i,*}^2)]\\\label{Tue2}
&=&\sum_{i=1}^K\mathbb{E}\tau_i\Gamma_{i,*}^2.
\end{eqnarray}
For a set of positive real numbers $h_i$, we have $\sqrt{\sum_{i} h_i}\le \sum_{i}\sqrt{h_i}$. We can apply this inequality to the third term on the RHS of~\eqref{Tue1}
and from~\eqref{Tue1},~\eqref{Tue2} and $\frac{1}{T}\sum_{i=1}^K\mathbb{E}\tau_i\sigma_i^2\le\sigma_{\max}^2$ (where $\sigma_{\max}=\max_i{\sigma_i}$) , we have
\begin{eqnarray}\label{kkks0}
\frac{1}{T}\mathbb{E}[(\sum_{i=1}^K\tau_i(\overline{\mu}_i-\mu_*))^2]\le \sigma_{\max}^2+\sum_{i=1}^K\mathbb{E}\tau_i\Gamma_{i,*}^2+2\sigma_{\max}\sum_{i=1}^K\sqrt{\mathbb{E}\tau_i\Gamma_{i,*}^2}.
\end{eqnarray}

Thus we can write
\begin{eqnarray}\nn
\widehat{R}_{\pi^*}(T)&\ge& \sum_{i=1}^K\mathbb{E}\tau_i\Delta_i+\sum_{i=1}^K\mathbb{E}\tau_i\Gamma_{i,*}^2-\sum_{i=1}^K\mathbb{E}\tau_i\Gamma_{i,*}^2-2\sigma_{\max}\sqrt{\sum_{i=1}^K\mathbb{E}\tau_i\Gamma_{i,*}^2}+\sigma_*^2-\sigma_{\max}^2\\\nn
&\ge&\sum_{i=1}^K\mathbb{E}\tau_i\Delta_i-2\sigma_{\max}\sum_{i=1}^K\sqrt{\mathbb{E}\tau_i}|\Gamma_{i,*}|-\sigma_{\max}^2\\\nn
&\ge&\sum_{i=1}^K\min_{x\ge0}(x^2\Delta_i-2\sigma_{\max}|\Gamma_{i,*}|x)-\sigma_{\max}^2\\\nn
&=&-\sum_{i\neq*}\frac{\sigma_{\max}^2\Gamma_{i,*}^2}{\Delta_i}-\sigma_{\max}^2.
\end{eqnarray}
This gives a lower bound on $\widehat{R}_{\pi^*}(T)$ which translates to an upper bound on $R_{\pi}(T)-\widehat{R}_{\pi}(T)$ by~\eqref{56en}.
Although this lower bound is a constant independent of $T$, it grows unboundedly when $\Delta$ approaches 0. We next drive another lower bound on $\widehat{R}_{\pi^*}(T)$ that is independent of $\Delta$.

\begin{eqnarray}\nn
\widehat{R}_\pi(T) &=&\sum_{i=1}^K\mathbb{E}\tau_i\Delta_i+\sum_{i\neq*}\mathbb{E}\tau_i\Gamma_{i,*}^2-\frac{1}{T}\mathbb{E}[(\sum_{i=1}^K\tau_i(\overline{\mu}_i-\mu_*))^2]+\sigma_*^2\\\label{tue3}
&\ge&\sum_{i=1}^K\mathbb{E}\tau_i\Delta_i+\sum_{i\neq*}\mathbb{E}\tau_i\Gamma_{i,*}^2-\sum_{i=1}^K\mathbb{E}[\tau_i(\overline{\mu}_i-\mu_*)^2]\\\nn
&=&\sum_{i=1}^K\mathbb{E}\tau_i\Delta_i+\sum_{i\neq*}\mathbb{E}\tau_i\Gamma_{i,*}^2-\sum_{i=1}^K\mathbb{E}[\tau_i(\overline{\mu}_i-\mu_i+\Gamma_{i,*})^2]\\\nn
&=&\sum_{i=1}^K\mathbb{E}\tau_i\Delta_i-\sum_{i=1}^K\mathbb{E}[\tau_i(\overline{\mu}_i-\mu_i)^2]\\\label{kkks1}
&\ge&\sum_{i=1}^K\mathbb{E}\tau_i\Delta_i-\frac{K}{a}\log T\\\nn
&\ge&-\frac{K}{a}\log T.
\end{eqnarray}
Inequality~\eqref{tue3} holds as a result of Cauchy-Schwartz inequality (similar to~\eqref{Tue2}) and~\eqref{kkks1} holds by Lemma~\ref{XS1}.
\subsection*{Appendix E: Proof of Lemma~\ref{LowMSPr} }%\label{AppB}
For $k\neq*$, construct $\mathcal{F}^{k}$, by only changing the distribution of arm $k$ to $f'_k$ such that arm $k$ is the optimal arm ($-\delta<\xi'_k-\xi_*<0$) and $|I(f_k,f_*)- I(f_k,f'_k)|\le \epsilon$ for arbitrary small $\epsilon$. The possibility of such a model is a result of Assumption 1.
%Let $\nu$ and $\nu^{k}$, denote the probability measures corresponding to the original and newly constructed  models.
Let $\gamma$ denote the log-likelihood ratio between the $\mathcal{F}$ and $\mathcal{F}^k$: $\gamma=\log \frac{f_k(X_k(t_k(1)))...f_k(X_k(t_k(\tau_k)))}{f'_k(X_k(t_k(1)))...f'_k(X_k(t_k(\tau_k)))}$. We show that it is unlikely to have $\tau_k<\frac{c_1\log T}{I(f_k,f'_k)}$ under two different scenarios for $\gamma$.

First, consider $\gamma>c_5\log T$ for a constant $c_5>c_1$. We have
\begin{eqnarray}\nn
\mathbb{P}_{\mathcal{F}}[\tau_k<\frac{c_1\log T}{I(f_k,f'_k)}, \gamma> c_5\log T]&=& \mathbb{P}_{\mathcal{F}}[\tau_k<\frac{c_1\log T}{I(f_k,f'_k)}, \sum_{s=1}^{\tau_k}\log\frac{f_k(X_k(s))}{f'_k(X_k(s))}> c_5\log T]\\\nn
&\le&\mathbb{P}_{\mathcal{F}}[ \max_{t\le \frac{c_1\log T}{I(f_k,f'_k)}}\sum_{s=1}^{t}\log\frac{f_k(X_k(s))}{f'_k(X_k(s))}> \frac{c_5\log T}{I(f_k,f'_k)}I(f_k,f'_k)]\\\nn
&\le&\mathbb{P}_{\mathcal{F}}[ \max_{t\le \frac{c_1\log T}{I(f_k,f'_k)}}\frac{1}{\frac{c_1\log T}{I(f_k,f'_k)}}\sum_{s=1}^{t}\log\frac{f_k(X_k(s))}{f'_k(X_k(s))}> \frac{c_5}{c_1}I(f_k,f'_k)].
\end{eqnarray}
By strong law of large numbers $\frac{1}{t}\sum_{s=1}^{t}\log\frac{f_k(X_k(s))}{f'_k(X_k(s))}\rightarrow I(f_k,f'_k)$ $a.s.$ as $t\rightarrow \infty$. Notice that $\mathbb{E}_{\mathcal{F}}[\log\frac{f_k(X_k(s))}{f'_k(X_k(s))}]=I(f_k,f'_k)$. Thus,\\ $\max_{t\le \frac{c_1\log T}{I(f_k,f'_k)}}\frac{1}{\frac{c_1\log T}{I(f_k,f'_k)}}\sum_{s=1}^{t}\log\frac{f_k(X_k(s))}{f'_k(X_k(s))}\rightarrow I(f_k,f'_k)$ $a.s.$ as $T\rightarrow \infty$. We thus have
\begin{eqnarray}\nn
\mathbb{P}_{\mathcal{F}}[ \max_{t\le \frac{c_1\log T}{I(f_k,f'_k)}}\frac{1}{\frac{c_1\log T}{I(f_k,f'_k)}}\sum_{s=1}^{t}\log\frac{f_k(X_k(s))}{f'_k(X_k(s))}> \frac{c_5}{c_1}I(f_k,f'_k)]\rightarrow 0,~\textit{as}~T\rightarrow\infty.
\end{eqnarray}
For $\gamma> c_5\log T$, by strong law of large numbers, we have
\begin{eqnarray}\label{SLL}
\mathbb{P}_{\mathcal{F}}[\tau_k<\frac{c_1\log T}{I(f_k,f'_k)}, \gamma> c_5\log T]\rightarrow 0,~\textit{as}~T\rightarrow\infty.
\end{eqnarray}
Also, when Assumption 2 is satisfied, $\log f'_k(X)$ and $\log f_k(X)$ have sub-Gaussian distributions. Thus, $\log \frac{f'_k(X)}{f_k(X)}$ has sub-Gaussian distribution and using Chernoff-Hoeffding bound we can prove an upper bound for $\mathbb{P}_{\mathcal{F}}[\tau_k<\frac{c_1\log T}{I(f_k,f'_k)}, \gamma> c_5\log T]$, for finite $T$. Specifically,

{\small{\begin{eqnarray}\nn
\mathbb{P}_{\mathcal{F}}[\tau_k<\frac{c_1\log T}{I(f_k,f'_k)}, \gamma> c_5\log T]&=& \mathbb{P}_{\mathcal{F}}[\tau_k<\frac{c_1\log T}{I(f_k,f'_k)}, \sum_{s=1}^{\tau_k}\log\frac{f_k(X_k(s))}{f'_k(X_k(s))}> c_5\log T]\\\nn
&\le&\mathbb{P}_{\mathcal{F}}[\max_{t< \frac{c_1\log T}{I(f_k,f'_k)}}\sum_{s=1}^{t}\log\frac{f_k(X_k(s))}{f'_k(X_k(s))}> c_5\log T]\\\nn
&\le&\sum_{t=1}^{\frac{c_1\log T}{I(f_k,f'_k)}}\mathbb{P}_{\mathcal{F}}[\frac{1}{t}\sum_{s=1}^{t}\log\frac{f_k(X_k(s))}{f'_k(X_k(s))}-\frac{1}{t}c_1\log T>{\frac{1}{t}c_5\log T}-\frac{1}{t}c_1\log T ]\\\label{ed1}
&\le& \sum_{t=1}^{\frac{c_1\log T}{I(f_k,f'_k)}}\mathbb{P}_{\mathcal{F}}[\frac{1}{t}\sum_{s=1}^{t}\log\frac{f_k(X_k(s))}{f'_k(X_k(s))}-I(f_k,f'_k)>{\frac{1}{t}c_5\log T}-\frac{1}{t}c_1\log T ]~~~~~~\\\label{b2}
&\le&\sum_{t=1}^{\frac{c_1\log T}{I(f_k,f'_k)}}\exp(-a_1(c_5-c_1)^2\log^2 T/t)\\\label{NR1}
&\le&{\frac{c_1\log T}{I(f_k,f'_k)}}T^{-a_1I(f_k,f'_k)\frac{(c_5-c_1)^2}{c_1}}.
\end{eqnarray}}}
Inequality~\eqref{ed1} holds since $I(f_k,f'_k)\le\frac{1}{t}c_1\log T$ and~\eqref{b2} holds according to Chernoff-Hoeffding bound\footnote{For the locally sub-Gaussian distributions, a similar upper bound can be obtained based on~\eqref{nine01}. Specifically,~\eqref{NR1} will be replaced by ${\frac{c_1\log T}{I(f_k,f'_k)}}T^{-a_2(c_5-c_1)}$ where $a_2>0$ is a constant independent of $T$. The lemma can then be similarly proved.}. We point out that the Chernoff-Hoeffding bound constant $a_1$ is different from the constant specified in~\eqref{CHBnd} since we have a different random variable here.

%$\tcr[$In order to show that $\log\frac{f_k(X_k(s))}{f'_k(X_k(s))}$ has a light-tailed distribution we show that $\log f'_k(X_k(s))$ has a light tailed distribution. It is similar to shove that $\log f_k(X_k(s))$ has a light-tailed distribution and since the sum of two random variables with light-tailed distribution has a light-tailed distribution we conclude that $\log\frac{f_k(X_k(s))}{f'_k(X_k(s))}$ has a light-tailed distribution. Now, we prove that $\log f'_k(X_k(s))$ has a light tailed distribution. Let $M=\operatorname{ess}\sup f_k(x)$.
%\begin{eqnarray}\nn
%\mathbb{E}_{\mathcal{F}}[e^{u\log f'_k(X)}]&=&\mathbb{E}_{\mathcal{F}}[f'_k(X)^u]\\\nn
%&=&\int f'_k(x)^u f_k(x)dx\\\label{holder}
%&\le&(\int f'_k(x)dx)^{\frac{1}{u}}(\int f_k(x)^{\frac{1}{1-u}}dx)^{1-u}\\\nn
%&=&(\int f_k(x)^{\frac{1}{1-u}}dx)^{1-u}\\\nn
%&\le&M^{u}\\
%&\le&e^{\frac{\zeta u^2}{2}},
%\end{eqnarray}
%$\tcr]$

Next, we consider $\gamma\le c_5\log T$. By Markov inequality, we have

\begin{eqnarray}\nn
\mathbb{P}_{\mathcal{F}^{k}}[\tau_k<\frac{c_1\log T}{I(f_k,f'_k)}]&=& \mathbb{P}_{\mathcal{F}^{k}}[T-\tau_k\ge T-\frac{c_1\log T}{I(f_k,f'_k)}]\\\label{this}
&\le& \frac{\mathbb{E}_{\mathcal{F}^{k}}[T-\tau_k]}{T-\frac{c_1\log T}{I(f_k,f'_k)}}.
\end{eqnarray}
We can change the probability measure from $\mathcal{F}$ to $\mathcal{F}^k$ as follows. Let $\mathcal{S}(T)$ be the set of all observations over a time horizon with length $T$ that satisfy a particular event. We have
\begin{eqnarray}\nn
\mathbb{P}_{\mathcal{F}}[\mathcal{S}(T)]&=&\mathbb{E}_{\mathcal{F}}[\mathbb{I}_{\mathcal{S}(T)}]\\\nn
&=&\int_{\mathcal{S}(T)}\prod_{i=1}^K\prod_{s=1}^{\tau_i}f_i(x_i(t_i(s)))\prod_{i=1}^K\prod_{s=1}^{\tau_i}dx_i(t_i(s))\\\nn
&=&\int_{\mathcal{S}(T)}\prod_{i=1,i\neq k}^K\prod_{s=1}^{\tau_i}f_i(x_i(t_i(s)))\prod_{s=1}^{\tau_k}f'_k(x_i(t_i(s)))\frac{f_k(x_i(t_i(s)))}{f'_k(x_i(t_i(s)))}\prod_{i=1}^K\prod_{s=1}^{\tau_i}dx_i(t_i(s))\\\nn
&=&\mathbb{E}_{\mathcal{F}^k}[\mathbb{I}_{\mathcal{S}(T)}\Pi_{s=1}^{\tau_k}\frac{f_k(x_i(t_i(s)))}{f'_k(x_i(t_i(s)))}]\\\nn
&=&\mathbb{E}_{\mathcal{F}^k}[\mathbb{I}_{\mathcal{S}(T)}e^\gamma].
\end{eqnarray}

Using~\eqref{this} and a change of probability measure from $\mathcal{F}$ to $\mathcal{F}^k$ we have
\begin{eqnarray}\nn
\mathbb{P}_{\mathcal{F}}[\tau_k<\frac{c_1\log T}{I(f_k,f'_k)}, \gamma\le c_5\log T]&=& \mathbb{E}_{\mathcal{F}}[\mathbb{I}_{\{\tau_k<\frac{c_1\log T}{I(f_k,f'_k)}, \gamma\le c_5\log T\}}]\\\nn
&\le&\mathbb{E}_{\mathcal{F}^k}[\mathbb{I}_{\{\tau_k<\frac{c_1\log T}{I(f_k,f'_k)}, \gamma\le c_5\log T\}}e^\gamma]\\\nn
&\le& T^{c_5}\mathbb{E}_{\mathcal{F}^k}[\mathbb{I}_{\{\tau_k<\frac{c_1\log T}{I(f_k,f'_k)}, \gamma\le c_5\log T\}}]\\\nn
&\le& T^{c_5}\mathbb{P}_{\mathcal{F}^{k}}[\tau_k<\frac{c_1\log T}{I(f_k,f'_k)}]\\\nn
&\le& \frac{T^{c_5}\mathbb{E}_{  \mathcal{F}^k}[T-\tau_k]}{T-\frac{c_1\log T}{I(f_k,f'_k)}}\\\label{NR2}
&\le& \frac{KT^{c_5+\alpha}}{T-\frac{c_1\log T}{I(f_k,f'_k)}},
\end{eqnarray}
where to arrive at the last inequality we use the $\alpha-$consistency assumption.
Form~\eqref{SLL},~\eqref{NR2} and the fact that $|I(f_k,f_*)- I(f_k,f'_k)|$ can be arbitrarily small, we conclude that, for $c_5<1-\alpha$
\begin{eqnarray}\nn
\mathbb{P}_{\mathcal{F}}[\tau_k<\frac{c_1\log T}{I(f_k,f_*)}]\rightarrow0,~\textit{as}~T\rightarrow\infty.
\end{eqnarray}
Equivalently,
\begin{eqnarray}\nn
\mathbb{P}_{\mathcal{F}}[\tau_k\ge\frac{c_1\log T}{I(f_k,f_*)}]\rightarrow 1,~\textit{as}~T\rightarrow\infty.
\end{eqnarray}
Form~\eqref{NR1} and~\eqref{NR2}, we conclude that, for $c_5<1-\alpha$, when Assumption 2 is satisfied

\begin{eqnarray}\nn
\mathbb{P}_{\mathcal{F}}[\tau_k<\frac{c_1\log T}{I(f_k,f_*)}]\le {\frac{c_1\log T}{I(f_k,f'_k)}}T^{-a_1I(f_k,f'_k)\frac{(c_5-c_1)^2}{c_1}}+\frac{KT^{c_5+\alpha}}{T-\frac{c_1\log T}{I(f_k,f'_k)}}.
\end{eqnarray}
Thus, there is a $T_0\in \mathbb{N}$ such that for $T\ge T_0$,
\begin{eqnarray}\nn
\mathbb{P}_{\mathcal{F}}[\tau_k\ge\frac{c_1\log T}{I(f_k,f_*)}]\ge c_2.
\end{eqnarray}
for some constant $c_2>0$ independent of $T$ and $\mathcal{F}$. We emphasize that the constant $c_1$ and $c_5$ are chosen to satisfy $c_1<c_5<1-\alpha$.
%Thus,
%\begin{eqnarray}\nn
%\mathbb{E}_{\nu^1}[\tau_k]>\frac{c_1c_2\log T}{I(f_1,f_k)}.
%\end{eqnarray}

%%%%%%%%%%%%%%%%%%%%%%%%%%%%%%%%%%%%%%%%%%%%%%%%%%%%%%%%%%%%%%%%%%%%%%%%%%%%%%%%%%%%%%%%%%%%%%%%%%%%%%%%%%%%%%%%%%%%%%%%%%%%%%%%%%%%%
\subsection*{Appendix F: Proof of Theorem~\ref{MSL} }%\label{AppB}

Since $R_\pi(T) \ge \widehat{R}_\pi(T)$ we can establish a lower bound on $\widehat{R}_\pi(T)$. From Lemma~\ref{RegEx} we have
\begin{eqnarray}\label{regi}
\widehat{R}_\pi(T)
&=&  \sum_{i=1}^K\mathbb{E}\tau_i\Delta_i+\sum_{i=1}^K\mathbb{E}\tau_i\Gamma_{i,*}^2-\frac{1}{T}\mathbb{E}[(\sum_{i=1}^K\tau_i(\overline{\mu}_i-\mu_*))^2]+\sigma_*^2.
\end{eqnarray}
A lower bound on $\mathbb{E}[\tau_i]$ is a straightforward consequence of Lemma~\ref{LowMSPr}. By Markov inequality we have
\begin{eqnarray}
\mathbb{E}[\tau_i]\ge\mathbb{P}[\tau_i\ge\frac{c_1\log T}{I(f_i,f_*)}]\frac{c_1\log T}{I(f_i,f_*)}.
\end{eqnarray}
So we can write
\begin{eqnarray}\nn
\lim_{T\rightarrow\infty}\frac{\mathbb{E}[\tau_i]}{\log T}&\ge&\lim_{T\rightarrow\infty}\mathbb{P}[\tau_i\ge\frac{c_1\log T}{I(f_i,f_*)}]\frac{c_1}{I(f_i,f_*)}\\\nn
&=&\frac{c_1}{I(f_i,f_*)}.
\end{eqnarray}
Similarly, by Markov inequality we have, when Assumption 2 is satisfied, there is a $T_0\in \mathbb{N}$ such that for all $T\ge T_0$,
\begin{eqnarray}
\mathbb{E}[\tau_i]\ge \frac{c_1 c_2\log T}{I(f_i,f_*)}.
\end{eqnarray}

For the third term on the RHS of regret expression~\eqref{regi}, following the similar steps as in~\eqref{kkks0}, we have

\begin{eqnarray}\label{koonnn1}
\frac{1}{T}\mathbb{E}[(\sum_{i=1}^K\tau_i(\overline{\mu}_i-\mu_*))^2]
&\le&\sigma_{\max}^2+\frac{1}{T}\mathbb{E}[(\sum_{i=1}^K\tau_i\Gamma_{i,*})^2]+2\sigma_{\max}\sum_{i=1}^K\sqrt{\mathbb{E}\tau_i\Gamma_{i,*}^2}.
\end{eqnarray}

Define the event $\mathcal{E}$ as follows.
$\mathcal{E}:$ for all $k\neq*$, $\tau_k\le T^{\frac{1+\alpha}{2}}$.

For $\sum_{i=1}^K\mathbb{E}\tau_i\Gamma_{i,*}^2-\frac{1}{T}\mathbb{E}[(\sum_{i=1}^K\tau_i\Gamma_{i,*})^2]$, we have

\begin{eqnarray}\label{akh2}
\sum_{i=1}^K\mathbb{E}\tau_i\Gamma_{i,*}^2-\frac{1}{T}\mathbb{E}[(\sum_{i=1}^K\tau_i\Gamma_{i,*})^2]
&=&\sum_{i=1}^K\mathbb{E}[\tau_i\Gamma_{i,*}^2-\frac{1}{T}(\sum_{i=1}^K\tau_i\Gamma_{i,*})^2]\\\label{akh3}
&\ge&\sum_{i=1}^K\mathbb{E}[\tau_i\Gamma_{i,*}^2-\frac{1}{T}(\sum_{i=1}^K\tau_i\Gamma_{i,*})^2,\mathcal{E}]\\\nn
&=&\sum_{i\neq*}\mathbb{E}[\tau_i(\Gamma_{i,*}^2-\frac{1}{T}\sum_{j\neq*}\tau_j\Gamma_{i,*}\Gamma_{j,*}),\mathcal{E}]\\\nn
%&\ge&\sum_{i=1}^K\mathbb{E}[\tau_i(\Gamma_{i,*}^2-\frac{1}{T}\sum_{j=1}^K\tau_j\Gamma_{i,*}\Gamma_{j,*}),\mathcal{E}]\\\nn
&\ge&\sum_{i\neq*}\mathbb{E}[\tau_i(\Gamma_{i,*}^2-\frac{1}{T}(K-1)\Gamma^2T^{\frac{1+\alpha}{2}}),\mathcal{E}]\\\label{akh5}
&=&\sum_{i\neq*}\mathbb{E}[\tau_i(\Gamma_{i,*}^2-(K-1)\Gamma^2T^{-\frac{1-\alpha}{2}}),\mathcal{E}].
\end{eqnarray}
Notice that~\eqref{akh3} holds because the argument inside the expectation in~\eqref{akh2} is always positive (similar to~\eqref{Tue2} by Cauchy-Schwartz inequality).

We also have
\begin{eqnarray}\nn
\mathbb{P}[\tau_i\ge\frac{c_1\log T}{I(f_i,f_*)}, \mathcal{E}]&=&\mathbb{P}[\tau_i\ge\frac{c_1\log T}{I(f_i,f_*)}]-\mathbb{P}[\tau_i<\frac{c_1\log T}{I(f_i,f_*)},\overline{ \mathcal{E}}]\\\nn
&\ge&\mathbb{P}[\tau_i\ge\frac{c_1\log T}{I(f_i,f_*)}]-\mathbb{P}[ \overline{ \mathcal{E}}]\\\nn
&\ge& \mathbb{P}[\tau_i\ge\frac{c_1\log T}{I(f_i,f_*)}]-\sum_{i\neq*} \mathbb{P}[\tau_i>T^{\frac{1+\alpha}{2}}]\\\nn
&\ge& \mathbb{P}[\tau_i\ge\frac{c_1\log T}{I(f_i,f_*)}]-\sum_{i\neq*} \frac{\mathbb{E}[\tau_i]}{T^{\frac{1+\alpha}{2}}}\\\label{gayid9}
&\ge&\mathbb{P}[\tau_i\ge\frac{c_1\log T}{I(f_i,f_*)}]-(K-1)T^{-\frac{1-\alpha}{2}}.
\end{eqnarray}
Thus, by Markov inequality
\begin{eqnarray}\nn
\mathbb{E}[\tau_i,\mathcal{E}]&\ge& \mathbb{P}[\tau_i\ge\frac{c_1\log T}{I(f_i,f_*)}, \mathcal{E}]\frac{c_1\log T}{I(f_i,f_*)}\\\nn
&\ge&\mathbb{P}[\tau_i\ge\frac{c_1\log T}{I(f_i,f_*)}]\frac{c_1\log T}{I(f_i,f_*)}-(K-1)T^{-\frac{1-\alpha}{2}}\frac{c_1\log T}{I(f_i,f_*)}
\end{eqnarray}
As a result of Lemma~\ref{LowMSPr}, we have
\begin{eqnarray}\nn
\lim_{T\rightarrow\infty}\frac{\mathbb{E}[\tau_i,\mathcal{E}]}{\log T}&\ge&\lim_{T\rightarrow\infty}\mathbb{P}[\tau_i\ge\frac{c_1\log T}{I(f_i,f_*)}]\frac{c_1}{I(f_i,f_*)}\\\nn
&&~~~-\lim_{T\rightarrow\infty}\frac{c_1(K-1)T^{-\frac{1-\alpha}{2}}}{I(f_i,f_*)}\\\nn
&=&\frac{c_1}{I(f_i,f_*)}.
\end{eqnarray}
Similarly, we have, when Assumption 2 is satisfied, there is a $T_0\in \mathbb{N}$ such that for all $T\ge T_0$,
\begin{eqnarray}
\mathbb{E}[\tau_k,\mathcal{E}]\ge \frac{c_1 c_2\log T}{I(f_i,f_*)}-(K-1)T^{-\frac{1-\alpha}{2}}\frac{c_1\log T}{I(f_i,f_*)}.
\end{eqnarray}

Substituting~\eqref{koonnn1} and~\eqref{akh5} in regret expression we have
\begin{eqnarray}\nn
\widehat{R}_\pi(T)&\ge& \sum_{i=1}^K\mathbb{E}\tau_i\Delta_i +\sum_{i\neq*}\mathbb{E}[\tau_i,\mathcal{E}](\Gamma_{i,*}^2-(K-1)\Gamma^2T^{-\frac{1-\alpha}{2}})\\\nn
&&~~~
-2\sigma_{\max}\sum_{i=1}^K\sqrt{\mathbb{E}\tau_i\Gamma_{i,*}^2}-\sigma_{\max}^2\\\nn
&=& \sum_{i=1}^K\mathbb{E}\tau_i(\Delta_i-\frac{2\sigma_{\max}|\Gamma_{i,*}|}{\sqrt{\mathbb{E}\tau_i}}-\frac{\sigma_{\max}^2}{\mathbb{E}\tau_i})\\\nn
&&~~+\sum_{i\neq*}\mathbb{E}[\tau_i,\mathcal{E}](\Gamma_{i,*}^2-(K-1)\Gamma^2T^{-\frac{1-\alpha}{2}}).
\end{eqnarray}
%We can write
%\begin{eqnarray}\nn
%\widehat{R}_\pi(T)&\ge& \sum_{i=1}^K\mathbb{E}\tau_i(\Delta_i-\frac{2\sigma_{\max}|\Gamma_{i,*}|}{\sqrt{\mathbb{E}\tau_i}}-\frac{\sigma_{\max}^2}{\mathbb{E}\tau_i})\\\nn
%&&~~+\sum_{i\neq*}\mathbb{E}[\tau_i,\mathcal{E}](\Gamma_{i,*}^2-(K-1)\Gamma^2T^{-\frac{1-\alpha}{2}}).
%\end{eqnarray}

Substituting the lower bounds on $\mathbb{E}[\tau_i]$ and $\mathbb{E}[\tau_i,\mathcal{E}]$ in the above bound we arrive at

\begin{eqnarray}
\lim\inf_{T\rightarrow\infty}\frac{\widehat{R}_\pi(T)}{\log T} &\ge& \sum_{i\neq*} \frac{c_1}{I(f_i,f_*)}(\Delta_i+\Gamma_{i,*}^2).
\end{eqnarray}
Also, when Assumption 2 is satisfied, for $T\ge T_0$,
\begin{eqnarray}\nn
\widehat{R}_\pi(T) &\ge& \sum_{i\neq*} \frac{c_1c_2\log T }{I(f_i,f_*)}(\Delta_i-\frac{2\sigma_{\max}|\Gamma_{i,*}|}{\sqrt{\frac{c_1c_2\log T }{I(f_i,f_*)}}}-\frac{\sigma_{\max}^2}{\frac{c_1c_2\log T }{I(f_i,f_*)}  })\\\nn
&&~~~+\sum_{i\neq*}(\frac{c_1 c_2\log T}{I(f_i,f_*)}-(K-1)T^{-\frac{1-\alpha}{2}}\frac{c_1\log T}{I(f_i,f_*)})\\\nn
&&~~~~~~~~~~~~~~~(\Gamma_{i,*}^2-(K-1)\Gamma^2T^{-\frac{1-\alpha}{2}}).
\end{eqnarray}
We can rewrite the above lower bound such that for all $T\ge T_1$,
\begin{eqnarray}
\widehat{R}_\pi(T) &\ge& \sum_{i\neq*} \frac{c_1c_2\log T }{I(f_i,f_*)}(\Delta_i+\Gamma_{i,*}^2-\epsilon_{T_1}),
\end{eqnarray}
where $\epsilon_{T_1}$ can be arbitrary small when $T_1$ is large enough. However, precise characterization of $\epsilon_{T_1}$ is tedious and depends on all of the diminishing terms above.

\subsection*{Appendix G: Proof of Lemma~\ref{tauL} }

Let $i\neq *$ be a suboptimal arm and $b_i=\lceil \frac{4b^2 \log T}{\min\{\Delta_i^2,4(2+\rho)^2\}}\rceil$.

\begin{eqnarray}\nn
\tau_i(T) &=& \sum_{t=1}^T \mathbb{I}[\pi(t)=i]\\\nn
%&=& 1+\sum_{t=K+1}^T \mathbb{I}[\pi(t)=i]\\\nn
&=& b_i+\sum_{t=b_i+1}^T \mathbb{I}[\pi(t)=i, \tau_i(t)\ge b_i]\\
&\le& b_i+\sum_{t=b_i+1}^T \mathbb{I}[\eta_i(t)- \eta_*(t)\le 0, \tau_i(t)\ge b_i].\label{nesfe}
\end{eqnarray}
We can write
\begin{eqnarray}\nn
\eta_i(t)- \eta_*(t) &=& \overline{\xi}_i(t)-b\sqrt{\frac{\log t}{\tau_i (t)}} - \eta_*(t)\\\nn
&=&(\overline{\xi}_i(t)+b\sqrt{\frac{\log t}{\tau_i (t)}} -\xi_i)-(\eta_*(t)-\xi_*)\\
&+& (\xi_i-\xi_*-2b\sqrt{\frac{\log t}{\tau_i (t)}}~).\label{bi}
\end{eqnarray}
For $\tau_i(t)\ge b_i$, the last term in~\eqref{bi} is positive, thus continuing from~\eqref{nesfe}
\begin{eqnarray}\nn
\tau_i(T) &\le& b_i+\sum_{t=b_i+1}^T \mathbb{I}[\overline{\xi}_i(t)+b\sqrt{\frac{\log t}{\tau_i (t)}} -\xi_i\le 0, \tau_i(t)\ge b_i]\\\nn
&+&\sum_{t=b_i+1}^T \mathbb{I}[\eta_*(t)-\xi_* \ge 0 ].
\end{eqnarray}
Applying Lemma~\ref{LemmaChernoff}
\begin{eqnarray}\nn
\mathbb{E}[\tau_i(T)] &\le& b_i + 2\sum_{t=b_i+1}^T t\exp(-\frac{a b^2 \log t}{(2+\rho)^2})+ 2\sum_{t=b_i+1}^T t\exp(-\frac{a b^2 \log t}{(1+\rho)^2})\\\nn
&\le& \frac{4b^2 \log T}{\min\{\Delta_i^2,4(2+\rho)^2\}} + 1 + 4\int_{b_i}^\infty t^{-2} dt \\\nn
&=& \frac{4b^2 \log T}{\min\{\Delta_i^2,4(2+\rho)^2\}} + 1 + 4 b_i^{-1}\\
&\le& \frac{4b^2 \log T}{\min\{\Delta_i^2,4(2+\rho)^2\}} + 5.
\end{eqnarray}

\subsection*{Appendix H: Proof of Theorem~\ref{T2} }\label{AppF}
Considering the regret expression in Lemma~\ref{RegEx}
\begin{eqnarray}\nn
\widehat{R}_{MV-UCB}(T) &=& \sum_{i=1}^K\mathbb{E}\tau_i\Delta_i+\sum_{i=1}^K\mathbb{E}\tau_i\Gamma_{i,*}^2-\frac{1}{T}\mathbb{E}[(\sum_{i=1}^K\tau_i(\overline{\mu}_i-\mu_*))^2]+\sigma_*^2\\\nn
&\le&\sum_{i\neq*}\mathbb{E}\tau_i(\Delta_i+\Gamma_{i,*}^2)+\sigma_*^2\\\nn
&\le&\sum_{i\neq*}(\frac{4b^2 \log T}{\min\{\Delta_i^2,4(2+\rho)^2\}} + 5)(\Delta_i+\Gamma_{i,*}^2)+\sigma_*^2.
\end{eqnarray}
From Theorem~\ref{TheKtarin}, we have
\begin{eqnarray}\nn
&&R_{MV-UCB}(T) \le \sum_{i\neq*}(\frac{4b^2 \log T}{\min\{\Delta_i^2,4(2+\rho)^2\}} + 5)(\Delta_i+\Gamma_{i,*}^2)+\sigma_*^2+\min\{\sigma_{\max}^2(\sum_{i\neq*}\frac{\Gamma_{i,*}^2}{\Delta_i}+1),\frac{K}{a}\log T\}.
\end{eqnarray}

%%%%%%%%%%%%%%%%%%%%%%%%%%%%%%%%%%%%%%%%%%%%%%%%%%%%%%%%%%%%%%%%%%%%%%%%%%%%%%%%%%%%%%%%%%%%%%%%%%%%%%%%%%%%%%%%%%%%%%%%%%%%%%%%%%%%%%%%%%%%%%%%%%%%%%%%%

%%%%%%%%%%%%%%%%%%%%%%%%%%%%%%%%%%%%%%%%%%%%%%%%%%%%%%%%%%%%%%%%%%%%%%%%%%%%%%%%%%%%%%%%%%%%%%%%%%%%%%%%%%%%%%%%%%%%%%%%%%%%%%%%%%%%%%%%%%%%%%%%%%%%%%%%%

\subsection*{Appendix I: Proof of Theorem~\ref{Low1} }\label{AppB}

The following lemma is used in the proof of the theorem.
\begin{lemma}[\cite{BR}]\label{lm6}
Let $\nu_0$, $\nu_1$ be two probability measures supported on some set
$\mathcal{X}$, with $\nu_1$ absolutely continuous with respect to $\nu_0$. Then, for any measurable
function $ \phi : \mathcal{X} \rightarrow \{0, 1\}$,
\begin{eqnarray}
\mathbb{P}_{\nu_1}[\phi(X)=0]+\mathbb{P}_{\nu_0}[\phi(X)=1]\ge \frac{1}{2}\exp(-I(\nu_0,\nu_1)).
\end{eqnarray}

\end{lemma}

To prove Theorem~\ref{Low1}, two different sets of distributions are assigned to a two-armed bandit. Then it is shown that under at least one of these two sets of distributions~\eqref{20m} holds.
Consider a two-armed bandit. Let $f_1 \sim \mathcal{N}(\mu_1,\sigma_1^2)$, a normal distribution with mean $\mu_1=\frac{3}{4}$ and variance $\sigma_1^2=\frac{3}{16}-4\Delta^2$. Also, let $f_2\sim\mathcal{B}(p)$, a Bernolli distribution with $p=1/4+2\Delta$ and $f'_2\sim\mathcal{B}(q)$ with $q=1/4-2\Delta$. Denote $\mathcal{F} = (f_1, f_2$) and $\mathcal{F}' = (f_1, f'_2)$.
%Also denote by $\nu_t$ and $\nu'_t$ the law of observed rewards up to time $t$ under $\nu$ and $\nu'$ respectively.
For the simplicity of presentation let us assume $\rho=0$. Note that for the difference between the variance of above distributions we have $\sigma^2_{2}-\sigma^2_{1}=\Delta$ and $\sigma^2_{1}-\sigma'^2_{2}=\Delta$. Since $\widehat{R}_\pi(T)\le{R}_\pi(T)$ we can establish a lower bound on $\widehat{R}_\pi(T)$ that is also a lower bound on ${R}_\pi(T)$. From Lemma~\ref{RegEx}

%\begin{eqnarray}
%\max(\widetilde{R}_\pi(T,\rho),\widetilde{R}_\pi(T,\rho'))\ge \widetilde{R}_\pi(T,\rho)\ge \Delta \mathbb{E}_\rho \tau_2(T)
%\end{eqnarray}

{{\begin{eqnarray}\nn
\widehat{R}_\pi(T) &=& \sum_{i=1}^K\mathbb{E}\tau_i\Delta_i+\sum_{i=1}^K\mathbb{E}\tau_i\Gamma_{i,*}^2-\frac{1}{T}\mathbb{E}[(\sum_{i=1}^K\tau_i(\overline{\mu}_i-\mu_*))^2]+\sigma_*^2
\end{eqnarray}}}
Following the similar lines as the proof of~\eqref{kkks1} we show
{{\begin{eqnarray}\nn
\widehat{R}_\pi(T) &\ge& \sum_{i=1}^2\mathbb{E}\tau_i\Delta_i-\frac{2}{a}(\log T+2).
\end{eqnarray}}}

Using Lemma~\ref{lm6} through a coupling argument we establish a lower bound on the regret under one of the two systems.

Let us use the notations $R_\pi(T;\mathcal{F})$ and $R_\pi(T;\mathcal{F}')$ to distinguish between the regrets under distribution assignments $\mathcal{F}$ and $\mathcal{F}'$, respectively. Also, let ${f}^{(t)}$ and ${f}'^{(t)}$ denote the distribution of the reward process up to time $t$ under $\mathcal{F}$ and $\mathcal{F}'$, respectively. Spcifically,
\begin{eqnarray}\nn
f^{(t)}(x(1),x(2),...,x(t))=\Pi_{\{s: \pi(s)=1\}}f_1(x(s))\Pi_{\{s: \pi(s)=2\}}f_2(x(s))
\end{eqnarray}
and ${f}'^{(t)}$ is defined similarly under $\mathcal{F}'$.
%Notice that the third term in~\eqref{pr1} is always positive. Using Lemma~\ref{lm6}, we have
{{\begin{eqnarray}\nn
&&\hspace{-5em}\max(\widehat{R}_\pi(T;\mathcal{F}),\widehat{R}_\pi(T;\mathcal{F}'))\\\nn&\ge& \frac{1}{2}(\widehat{R}_\pi(T;\mathcal{F}))+\widehat{R}_\pi(T;\mathcal{F}')) \\\nn
&\ge&\frac{\Delta}{2}\sum_{t=1}^T(\mathbb{P}_\mathcal{F}[\pi(t)=2]+\mathbb{P}_{\mathcal{F}'}[\pi(t)=1])-\frac{2}{a}(\log T+ 2)\\
&\ge&\frac{\Delta}{2}\sum_{t=1}^{T}(\mathbb{P}_\mathcal{F}[\pi(t)=2]+\mathbb{P}_{\mathcal{F}'}[\pi(t)=1])-\frac{2}{a}(\log T+ 2)\\ \label{max2}
&\ge&\frac{\Delta}{4}\sum_{t=1}^{T}\exp(-I(f^{(t)},f'^{(t)}))-\frac{2}{a}(\log T+ 2).\label{56s}
\end{eqnarray}}}
The KL-divergence between $f^{(t)}$ and $f'^{(t)}$ equals to
{{\begin{eqnarray}\nn
I(f^{(t)},f'^{(t)})&=&\mathbb{E}_\mathcal{F}[\log{ \frac{\Pi_{\{s:\pi(s)=2\}}f_2(X_{\pi(s)}(s))}{\Pi_{\{s:\pi(s)=2\}}f'_2(X_{\pi(s)}(s) )} }]\\\nn
&=&\mathbb{E}_\mathcal{F}[\sum_{s=1}^{\tau_2(t)}(p\log{\frac{p}{q}}+(1-p)\log{\frac{1-p}{1-q}})]\\\nn\label{KL}
&=&\mathbb{E}_\mathcal{F}\tau_2(t)(p\log{\frac{p}{q}}+(1-p)\log{\frac{1-p}{1-q}})\\
%&\le&\mathbb{E}_\nu\tau_2(t)c\Delta^2\\\nn
&\le&\mathbb{E}_\mathcal{F}\tau_2(t)d_0\Delta^2.\label{57s}
\end{eqnarray}}}
for some constant $d_0$. Substituting~\eqref{57s} in~\eqref{56s}
{{\begin{eqnarray}\nn
&&{\hspace{-5em}}\max(\widehat{R}_\pi(T;\mathcal{F}),\widehat{R}_\pi(T;\mathcal{F}'))\\
&\ge& \frac{\Delta}{4}T\exp(-\mathbb{E}_\mathcal{F}\tau_2(T)d_0\Delta^2)-\frac{2}{a}(\log T+ 2).\label{58z}
\end{eqnarray}}}

Following the similar lines as the proof of~\eqref{Tue1}, we show that

\begin{eqnarray}\nn
\widehat{R}_\pi(T)&\ge& \sum_{i=1}^2\mathbb{E}\tau_i\Delta_i+\sum_{i=1}^2\mathbb{E}\tau_i\Gamma_{i,*}^2 -\frac{1}{T}\mathbb{E}[(\sum_{i=1}^2\tau_i\Gamma_{i,*})^2]\\\nn
&&~~~-\sigma_{\max}^2-2\sigma_{\max}\sqrt{\frac{1}{T}\mathbb{E}[(\sum_{i=1}^2\tau_i\Gamma_{i,*})^2]}.
\end{eqnarray}

We can write
\begin{eqnarray}\nn
\widehat{R}_\pi(T;\mathcal{F})
&\ge&\mathbb{E}_{\mathcal{F}}\tau_2\Delta+\mathbb{E}_{\mathcal{F}}\tau_2\Gamma^2-\frac{1}{T}\mathbb{E}_{\mathcal{F}}[\tau_2^2]\Gamma^2-\sigma_{\max}^2-2\sigma_{\max}\sqrt{\frac{1}{T}\mathbb{E}_{\mathcal{F}}[\tau_2^2]\Gamma^2}\\\label{kard}
&=&\mathbb{E}_{\mathcal{F}}\tau_2\Delta+\frac{1}{T}\mathbb{E}_{\mathcal{F}}[\tau_1\tau_2]\Gamma^2-\sigma_{\max}^2-2\sigma_{\max}\sqrt{\mathbb{E}_{\mathcal{F}}[\tau_2]\Gamma^2}.
\end{eqnarray}

For the first term on the right hand side of~\eqref{kard}, we have, for a constant $0<d_3<1$
{{\begin{eqnarray}\nn
\frac{1}{T}\mathbb{E}_{\mathcal{F}}[\tau_1(T) \tau_2(T) ]&\ge&\frac{T-d_3T}{T}\mathbb{E}_{\mathcal{F}}[\tau_2 ,\tau_2\le d_3T]\\\label{spokh}
%&\ge& \frac{T-d_3 T}{T}\mathbb{E}_{\mathcal{F}}[\tau_2]- \frac{T-d_3T}{T}\mathbb{E}_{\mathcal{F}}[\tau_2 ,\tau_2\ge d_3T]\\\label{spokh}
&\ge& \frac{1}{2}\frac{T-d_3 T}{T}\mathbb{E}_{\mathcal{F}}[\tau_2]\\\label{spokh73}
&=&d_4\mathbb{E}_{\mathcal{F}}[\tau_2],
\end{eqnarray}}}
where $d_4=\frac{1-d_3}{2}$.
To arrive at~\eqref{spokh}, notice that we have
\begin{eqnarray}
\mathbb{E}_{\mathcal{F}}[\tau_2 ,\tau_2> d_3T]\le\mathbb{E}_{\mathcal{F}}[\tau_2 ,\tau_2\le d_3T],
\end{eqnarray}
otherwise $\mathbb{E}_{\mathcal{F}}[\tau_2]$ is linear with time and we arrive at the theorem. We show that $\mathbb{E}_{\mathcal{F}}[\tau_2 ,\tau_2> d_3T]\le\mathbb{E}_{\mathcal{F}}[\tau_2 ,\tau_2\le d_3T]$ translates to $\mathbb{E}_{\mathcal{F}}[\tau_2, \tau_2\le d_3T]\ge \frac{1}{2}\mathbb{E}[\tau_2]$.
{{\begin{eqnarray}\nn
&&\hspace{-3em}\mathbb{E}_{\mathcal{F}}[\tau_2]-\mathbb{E}[\tau_2 ,\tau_2> d_3T]\\\nn
&=&\mathbb{E}_{\mathcal{F}}[\tau_2 ,\tau_2\le d_3T]+\mathbb{E}_{\mathcal{F}}[\tau_2 ,\tau_2>d_3T]-\mathbb{E}_{\mathcal{F}}[\tau_2 ,\tau_2> d_3T]\\\nn
&\ge& \mathbb{E}_{\mathcal{F}}[\tau_2 ,\tau_2\le d_3T]+\mathbb{E}_{\mathcal{F}}[\tau_2 ,\tau_2> d_3T]-\frac{1}{2}\mathbb{E}_{\mathcal{F}}[\tau_2 ,\tau_2> d_3T]-\frac{1}{2}\mathbb{E}_{\mathcal{F}}[\tau_2 ,\tau_2\le d_3T]\\\nn
&=&\frac{1}{2}(\mathbb{E}_{\mathcal{F}}[\tau_2 ,\tau_2\le d_3T]+\mathbb{E}_{\mathcal{F}}[\tau_2 ,\tau_2> d_3T])\\\nn
&=&\frac{1}{2}\mathbb{E}_{\mathcal{F}}[\tau_2].
\end{eqnarray}}}

By~\eqref{kard} and~\eqref{spokh73}, we can write
\begin{eqnarray}\nn
\max(\widehat{R}_\pi(T;\mathcal{F}),\widehat{R}_\pi(T;\mathcal{F}'))&\ge&\widehat{R}_\pi(T;\mathcal{F})\\\label{khah1}
&\ge&\mathbb{E}_{\mathcal{F}}\tau_2\Delta+d_4\mathbb{E}_{\mathcal{F}}\tau_2\Gamma^2-\sigma_{\max}^2-2\sigma_{\max}\sqrt{\mathbb{E}_{\mathcal{F}}[\tau_2]\Gamma^2}.
\end{eqnarray}

For brevity of notation, let $x\triangleq\mathbb{E}_\mathcal{F}\tau_2$. From~\eqref{58z} and~\eqref{khah1} (by taking average over two lower bounds) we have, for $T\ge T_0$ for some large enough $T_0\in\mathbb{{N}}$
%Denote $\tau\triangleq\mathbb{E}_\nu\tau_2(T)$.
\begin{eqnarray}\nn
\max(R_\pi(T;\mathcal{F}),R_\pi(T;\mathcal{F}'))&\ge&\frac{1}{2}\{\frac{T\Delta}{4}\exp(-\mathbb{E}_\mathcal{F}\tau_2d_0\Delta^2)-\frac{2}{a}(\log T+2)+\mathbb{E}_{\mathcal{F}}\tau_2\Delta \\\nn &&~~~+d_4\mathbb{E}_{\mathcal{F}}\tau_2\Gamma^2-\sigma_{\max}^2-2\sigma_{\max}\sqrt{\mathbb{E}_{\mathcal{F}}\tau_2\Gamma^2}\}\\\nn
&\ge&\min_{x\ge0}\frac{1}{2}\{\frac{T\Delta}{4}\exp(-x d_0\Delta^2)+d_4x\Gamma^2\}\\\nn
&&~~~+\min_{x\ge0}\frac{1}{2}\{x\Delta-2\sigma_{\max}\sqrt{x\Gamma^2}\}\\\nn
&&~~~-\frac{1}{a}(\log T+2)-\frac{1}{2}\sigma_{\max}^2\\\nn
&=&\frac{d_2\Gamma^2}{2d_0\Delta^2}(\log\frac{T d_0\Delta^3}{4d_4\Gamma^2}+1)\\\label{60cd}
&&~~~-\frac{\sigma_{\max}^2\Gamma^2}{2\Delta}-\frac{1}{a}(\log T+2)-\frac{1}{2}\sigma_{\max}^2\\\nn
&\ge&\frac{c_4\log T}{\Delta^2}.
\end{eqnarray}

Substituting $\Delta$ with $d_6T^{-1/3}$ in~\eqref{60cd} for a constant $d_6$ that satisfies $\frac{d_0d_6^3}{4d_4\Gamma^2}>1$, we have, for some constant $c_3>0$,
\begin{eqnarray}
R_\pi(T) \ge c_3 T^{2/3}.
\end{eqnarray}
\par In this proof, for the purpose of presentation, it is assumed $\rho=0$. For $\rho\neq 0$ the same proof holds with modified assignment of distributions. The assignment of distributions are as follows.
For $\rho\neq\frac{1}{2}$, let $p=\frac{1}{4}+2\delta$, $q=\frac{1}{4}-2\delta$, $\mu_1=\frac{3}{4}$ and $\sigma_1^2=\frac{3}{16}-4\delta^2+\frac{\rho}{2}$. For $\rho=\frac{1}{2}$, let $p=\frac{1}{3}+3\delta$, $q=\frac{1}{3}-3\delta$, $\mu_1=\frac{5}{6}$ and $\sigma_1^2=\frac{17}{36}-9\delta^2$.

\begin{biography}%
[{\includegraphics[width=0.8in]{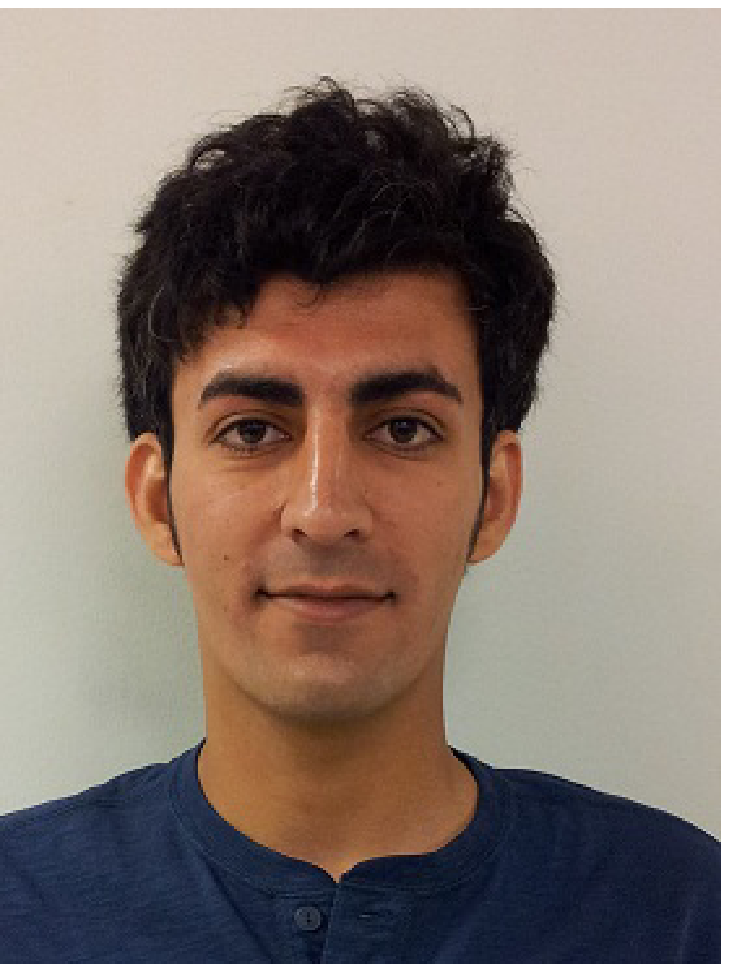}}]%<-- author's photograph
                                            % Uncomment if
                                            % you need a photograph
{Sattar Vakili}%<-- Name of author 1
received the B.S. degree in electrical engineering from Sharif University of Technology, Tehran, in 2011, and M.S. degree in electrical and computer engineering from University of California, Davis, in 2013.
He is currently a Ph.D. student at Cornell University. His research interests include stochastic online learning, decision theory, and communication and social-economic networks.
\end{biography}

\begin{biography}%
[{\includegraphics[width=0.8in]{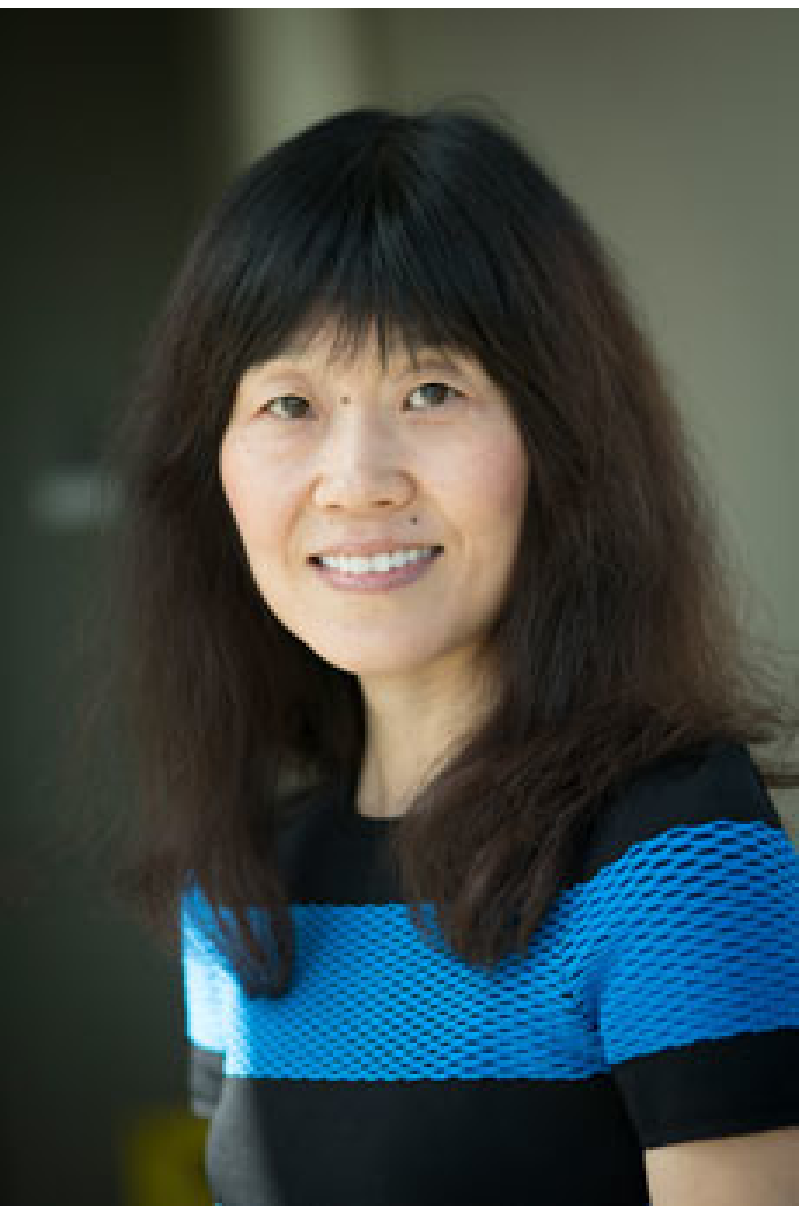}}]%<-- author's photograph
                                            % Uncomment if
                                            % you need a photograph
{Qing Zhao}%<-- Name of author 1
(F'13) joined the School of Electrical and Computer Engineering at Cornell University in 2015 as a Professor. Prior to that, she was a Professor at University of California, Davis. She received the Ph.D. degree in Electrical Engineering in 2001 from Cornell University. Her research interests are in the general area of stochastic optimization, decision theory, machine learning, and algorithmic theory in dynamic systems and communication and social-economic networks. She received the 2010 IEEE Signal Processing Magazine Best Paper Award and the 2000 Young Author Best Paper Award from the IEEE Signal Processing Society. \end{biography}

\end{document}

%% file: macros.tex
\def\boxit#1{\vbox{\hrule\hbox{\vrule\kern3pt
        \vbox{\kern3pt#1\kern3pt}\kern3pt\vrule}\hrule}}

\def\reals{ { {\rm  I \kern-0.15em R }  } }
\def\complex{ {\,{{\rm C} \kern-0.50em \raise0.20ex {  |}}\, }}

\def\Rbf{{\bf R}}

\def\Fc{{\cal F}}

\def\Uc{{\cal U}}

\def\be{\begin{equation}}
\def\ee{\end{equation}}

\def\scalefig#1{\epsfxsize #1\textwidth}
\def\Rxx{\Rbf_{\ssstyle X\kern-.1em X}}

\let\ssstyle=\scriptscriptstyle

\def\etal{{\it et al. \/}}

\def\Kout{\setbox1=\hbox{\Huge\bf K}\hbox to
1.05\wd1{\hspace{.05\wd1}
}}
\def\Sout{\setbox1=\hbox{\Huge\bf S}\hbox to 1.05\wd1{\hspace{.05\wd1}
}}